\documentclass[Afour,sageh,times]{sagej}

\usepackage{moreverb,url}

\usepackage[colorlinks,bookmarksopen,bookmarksnumbered,citecolor=red,urlcolor=red]{hyperref}
\usepackage{bm}
\usepackage{amsmath,amsfonts}

\usepackage{algorithmic}
\usepackage{algorithm}
\usepackage{array}
\usepackage{textcomp}
\usepackage{stfloats}
\usepackage{url}
\usepackage{appendix}

\usepackage{verbatim}
\usepackage{graphicx}
\usepackage{xcolor}

\usepackage[process=all]{pstool}
\usepackage{siunitx}
\usepackage{soul}
\usepackage{subcaption}

\makeatletter
\def\@thmheadfont{\bfseries}
\makeatother

\newtheorem{corollary}{Corollary}
\newtheorem{remark}{Remark}
\newtheorem{lemma}{Lemma}
\newtheorem{theorem}{Theorem}

\makeatletter
\renewenvironment{lemma}[1][]{%
  \refstepcounter{lemma}%
  \trivlist
  \item[\hskip\labelsep\bfseries Lemma~\thelemma%
    \if\relax\detokenize{#1}\relax.\else\ (#1).\fi]%
  \itshape
}{%
  \endtrivlist
}
\makeatother

\makeatletter
\renewenvironment{corollary}[1][]{%
  \refstepcounter{corollary}%
  \trivlist
  \item[\hskip\labelsep\bfseries Corollary~\thelemma%
    \if\relax\detokenize{#1}\relax.\else\ (#1).\fi]%
  \itshape
}{%
  \endtrivlist
}
\makeatother

\makeatletter
\renewenvironment{theorem}[1][]{%
  \refstepcounter{theorem}%
  \trivlist
  \item[\hskip\labelsep\bfseries Theorem~\thelemma%
    \if\relax\detokenize{#1}\relax.\else\ (#1).\fi]%
  \itshape
}{%
  \endtrivlist
}
\makeatother

\newcommand\BibTeX{{\rmfamily B\kern-.05em \textsc{i\kern-.025em b}\kern-.08em
T\kern-.1667em\lower.7ex\hbox{E}\kern-.125emX}}

\def\volumeyear{2026}
\begin{document}

\runninghead{Feliu-Talegon and Della Santina}

\title{Manipulation with Stability Guarantees: Linear Deformable Objects with Non-negligible \textbf{Physical} Response Grasped at Multiple Location}

\author{Daniel Feliu-Talegon\affilnum{1} and Cosimo Della Santina\affilnum{1}\affilnum{2}}

\affiliation{\affilnum{1}Department of Cognitive Robotics, Delft University of Technology, 2628 CN Delft, The Netherlands\\
\affilnum{2}Institute of Robotics and Mechatronics, German Aerospace Center (DLR), 82234 Oberpfaffenhofen, Germany}

\corrauth{Daniel Feliu-Talegon, Department of Cognitive Robotics, Delft University of Technology}

\email{d.feliutalegon@tudelft.nl}

\begin{abstract}
Most research on the manipulation of deformable objects focuses on lightweight systems with negligible mechanical response, effectively restricting attention to quasi-static regimes. This assumption excludes a broad class of practically relevant objects, such as hoses, pipes, and wiring harnesses, whose dynamics cannot be ignored during manipulation. In this work, we address this limitation by introducing a closed-loop control architecture that explicitly accounts for object dynamics and recasts manipulation as a shape-regulation problem. Control is achieved by modulating forces and torques applied at multiple fixed points along the object. This approach builds on three methodological contributions: a fully dynamic model of linear deformable objects based on discrete strain parameterizations; an extension of the notion of actuation coordinates to SE$(3)$, yielding a structured and inherently underactuated control architecture; and nonlinear feedback strategies providing explicit conditions for steady-state convergence to desired configurations. Extensive simulations on representative manipulation tasks demonstrate the performance gains enabled by the proposed model-based formulation. We finally validate the approach experimentally through a real-time closed-loop implementation with online shape estimation, confirming its practical feasibility and effectiveness.
\end{abstract}


\maketitle

\section{INTRODUCTION}

The integration of robots into human environments remains bottlenecked by manipulation, and in particular by the inability to robustly handle deformable objects—such as cables, hoses, or groceries—that pervade everyday tasks \cite{arriola2020modeling,yin2021modeling,fiorini2022concepts}. Deformable Linear Objects (DLOs)—including cables, wires, ropes, plants, and sutures—appear across a wide range of applications, from automotive, agricultural, and aerospace manufacturing to household and medical technologies. These objects are slender deformable bodies whose configuration evolves primarily along a single spatial dimension, making them prone to large, continuous deformations during interaction. 
As a result, even representing the state of a DLO requires, in principle, a theoretically infinite number of degrees of freedom, making its analysis inherently challenging. This difficulty is further exacerbated when object dynamics cannot be neglected—either because stiffness or mass are significant, or because the object undergoes rapid motion. Accordingly, as discussed in Sec.~\ref{sec:related}, most existing approaches focus either on model-free learning strategies or on quasi-static, often purely geometric, formulations. In both cases, the object's physical response is not explicitly accounted for. Consequently, achieving reliable robotic manipulation of DLOs with guaranteed stability and precision remains an open problem.

\begin{figure}[t]
\centerline{\includegraphics[width=0.49\textwidth]{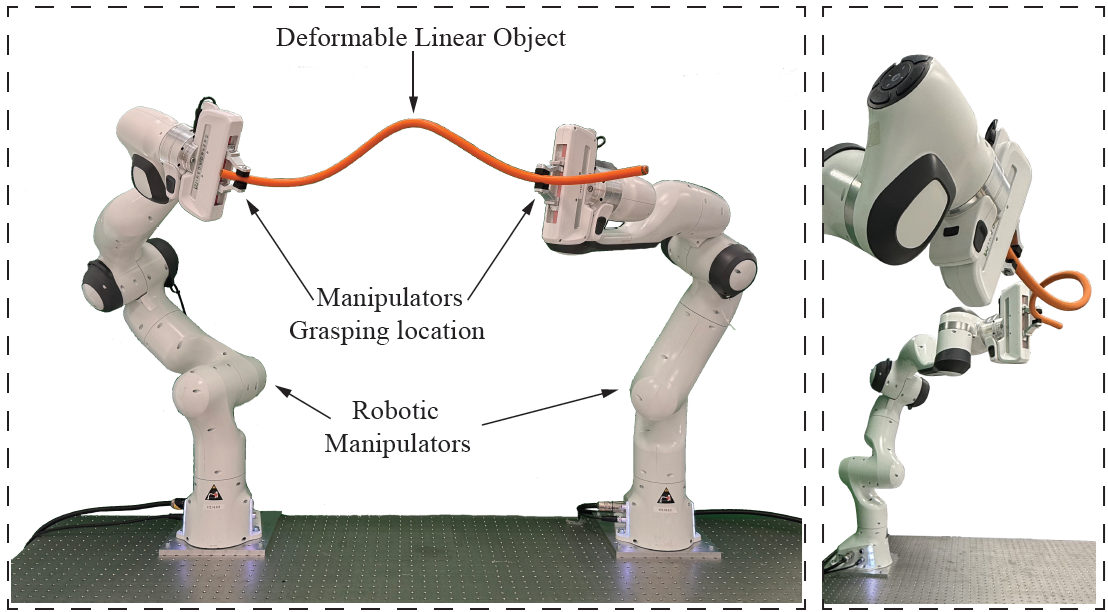}}
\caption{\small Manipulation of a deformable linear object (DLO) whose behavior is governed by its physical response and interaction with multiple contact points. Two robotic manipulators apply time-varying wrenches that couple with the object’s distributed inertia, elasticity, and gravity, driving the system toward the steady-state configuration shown, characterized by a nontrivial combination of bending, twisting, and shearing. This shape cannot be captured by a purely geometric description, but instead emerges from the coupled effects of internal elastic forces, gravitational loading, and externally applied wrenches. The proposed architecture achieves stable and precise three-dimensional shape control by treating the deformable object itself as the controlled dynamical system and regulating it through model-based feedback at a finite set of contact points.}
\label{fig:setup}
\end{figure}

This work introduces the first model-based closed-loop control framework for the robotic manipulation of generic DLOs deforming in three dimensions under the action of multiple contact points (see Fig.~\ref{fig:setup}). Manipulation is cast as a regulation problem in which the object itself defines the controlled dynamical system, while the control inputs are wrenches (forces and torques) applied at the contact points. This perspective enables precise shape control with formal guarantees of closed-loop stability. At the core of the approach lies a general dynamic model of DLOs as continuously deformable rods. Building on recent advances in soft robotics \cite{armanini2023soft,della2023model}, the model explicitly captures distributed inertia, elastic and gravitational effects, and treats the object as a floating-base system subject to wrenches applied at multiple, arbitrary locations. A strain-based discretization is then used to obtain accurate yet low-dimensional approximations of the resulting dynamics.

On this basis, we develop a general control framework with the following components:
\begin{enumerate}
\item We extend the theory of collocated coordinate transformations from $\mathbb{R}^m$ to $\mathrm{SE}(3)\times\cdots\times\mathrm{SE}(3)$ by introducing an intermediate representation of actuation actions. This allows the object coordinates to be systematically separated into directions directly driven by contact wrenches and directions evolving according to the internal dynamics. The resulting construction is general and applies beyond object manipulation to mechanical systems subject to three-dimensional torques.
\item Leveraging the physical insight provided by the proposed model, we introduce an optimization-based algorithm that selects task-consistent continuum three-dimensional shapes, enabling objectives such as positioning or orienting specific portions of the object in space.
\item Using the new coordinate representation of the object dynamics, we design multiple collocated controllers and derive explicit conditions ensuring convergence to the desired equilibrium shape.
\item We validate the proposed framework through an extensive set of simulations and experiments, detailing a complete two-arm manipulation pipeline and its real-time implementation, and demonstrating robust dynamic shape control in regimes where object stiffness, mass, and fast motions invalidate quasi-static assumptions.
\end{enumerate}

 \paragraph{Structure of the paper} 
 Section \ref{sec:related} reviews the most relevant literature on DLO manipulation. In Section \ref{sec:Collocation_Wrenches}, we tackle the problem of decoupling the actuated and unactuated dynamics of a Lagrangian system subject to multiple wrench inputs, a key step toward addressing model-based DLO manipulation. Section \ref{dynamic_model} introduces the DLO dynamic model, accounting for its non-negligible physical response (shown in Fig.~2 as Physics-Based Modelling), which underpins the design of both high-level and low-level controllers. Section \ref{sec:Model_based_Control} then presents the model-based control approach with stability guarantees proposed in this work, where a high-level controller generates a desired shape from the task objective, while a low-level controller computes the wrench inputs required to regulate the object shape in closed loop, as summarized in Fig.~\ref{fig:2_2}. Sections \ref{sec:simulation_results} and \ref{sec:experiments} provide extensive simulation and experimental results validating the proposed framework presented in Fig.~\ref{fig:2_2}. Finally, Section \ref{sec:conclusions} concludes the paper.

 \paragraph{Notation} we use right-handed coordinate frames: the fixed \emph{spatial} (or \emph{global}) frame, oriented upward opposite to gravity, and the moving \emph{body} frame attached to the system. Unless otherwise specified, the Jacobian, angular velocity, and velocity twist are expressed in the \emph{body} frame; quantities expressed in the \emph{spatial} frame are denoted with a superscript 
$\mathrm{s}$. Equilibrium quantities are marked with an overline, desired quantities with an asterisk. Bold symbols represent vectors and matrices, non-bold symbols denote scalars, and all vectors are \emph{column} vectors.
Rotations between frames are represented by a rotation matrix $\bm{R} \in SO(3)$, parameterized by ZYX Euler angles (yaw--pitch--roll), collected in $\bm{z} = [\alpha\ \beta\ \gamma]^{T}$. Positions with respect to the \emph{spatial} frame are denoted by $\bm{p} = [X\ Y\ Z]^{T}$. For quantities combining angular and linear components (e.g., wrenches, strain, and twists), angular components are ordered before linear ones.

\begin{table}[t]
\centering
\caption{\small List of symbols and their descriptions.}
\label{tab:symbols}
\footnotesize
\setlength{\tabcolsep}{4pt}
\renewcommand{\arraystretch}{1.1}
\begin{tabular}{l l p{0.48\columnwidth}}
\hline
\textbf{Symbol} & \textbf{Domain} & \textbf{Description} \\
\hline
$n$ & $\mathbb{N}$ & System degrees of freedom \\
$n_\mathrm{a}$ & $\mathbb{N}$ & Number of independent inputs \\
$n_o$ & $\mathbb{N}$ & Deformation DoFs of the DLO \\
$m$ & $\mathbb{N}$ & Task-space dimension \\
$s$ & $[0,1]\subset\mathbb{R}$ & Curvilinear abscissa along backbone \\
$\bm{q}, \dot{\bm{q}}, \ddot{\bm{q}}$ & $\mathbb{R}^n$ & Generalized coordinates, derivatives \\
$\bm{\theta}, \dot{\bm{\theta}}, \ddot{\bm{\theta}}$ & $\mathbb{R}^{n}$ & Collocated coordinates, derivatives \\
$\bm{M}$ & $\mathbb{R}^{n\times n}$ & Mass matrix \\
$\bm{C}$ & $\mathbb{R}^{n\times n}$ & Coriolis and centrifugal matrix \\
$\bm{D}$ & $\mathbb{R}^{n\times n}$ & Damping matrix \\
$\bm{K}$ & $\mathbb{R}^{n\times n}$ & Stiffness matrix \\
$\bm{G}$ & $\mathbb{R}^{n}$ & Gravitational force vector \\
$\bm{A}$ & $\mathbb{R}^{n\times n_\mathrm{a}}$ & Actuation matrix \\
$\bm{u}$ & $\mathbb{R}^{n_\mathrm{a}}$ & Inputs \\
$\bm{\tau}_{\mathrm{a}}$ & $\mathbb{R}^{n}$ & Lagrangian generalized forces \\
$\bm{\xi}$ & $\mathfrak{se}(3)$ & Screw strain \\
$\bm{\xi}^*$ & $\mathfrak{se}(3)$ & Reference screw strain \\
$\bm{\Phi}_{\xi}$ & $\mathbb{R}^{6\times n_\mathrm{o}}$ & Strain basis matrix \\
$\bm{g}$ & $SE(3)$ & Homogeneous transformation \\
$\bm{\eta}$ & $\mathfrak{se}(3)$ & Velocity twist \\
$\bm{J}$ & $\mathbb{R}^{6\times n}$ & Geometric Jacobian of the DLO \\
$\bm{W}$ & $\mathbb{R}^{6}$ & Wrench vector \\
$\bm{x}$ & $\mathbb{R}^{m}$ & Task-space coordinates \\
$\bm{z}$ & $\mathbb{R}^{3}$ & Euler angles (yaw--pitch--roll) \\
$\bm{p}$ & $\mathbb{R}^{3}$ & Position in the global frame \\
$\bm{\omega}$ & $\mathbb{R}^{3}$ & Angular velocity \\
$\bm{P}(\bm{z})$ & $\mathbb{R}^{3\times3}$ & Mapping between $\dot{\bm{z}}$ and $\bm{\omega}$ \\
$\bm{\tau}, \bm{f}$ & $\mathbb{R}^{3}$ & Moment and force vectors \\
$\mathrm{Ad}(\cdot)$ & $SE(3)\!\rightarrow\!\mathbb{R}^{6\times6}$ & Adjoint map \\
$\mathrm{ad}(\cdot)$ & $\mathfrak{se}(3)\!\rightarrow\!\mathbb{R}^{6\times6}$ & Adjoint operator \\
\hline
\end{tabular}
\end{table}

\section{RELATED WORKS}\label{sec:related}

Despite significant progress in rigid object manipulation, advances in robotic manipulation of deformable objects with non-negligible dynamics remain limited, mainly due to the high dimensionality of their configuration and the complexity of the underlying dynamics. This section reviews the most relevant literature on DLOs, with a focus on two key aspects: DLO modeling and control for manipulation tasks.

A broad range of physical models for deformable linear objects (DLOs) has been proposed \cite{lv2020review,chen2024differentiable}, including mass–spring systems, elastic-rod formulations \cite{gazzola2018forward}, Finite Element Method models \cite{zimmermann2021dynamic}, static differential-geometry approaches \cite{wakamatsu2004static} and spline-based dynamic representations \cite{palli2020model}. However, their adoption in robotics is often limited by real-time constraints and the resulting high computational cost. To address this, reduced-DoF simplified models have been developed, especially in soft robotics \cite{armanini2023soft,alessi2024rod}, with rod-theory-based models providing an efficient and control-friendly ODE formulation through geometric discretization \cite{della2019control, mathew2025reduced, webster2010design}.  In parallel, data-driven approaches have been explored \cite{yu2022global, yan2021learning, lee2021sample}, but learning accurate global deformation models remains data-inefficient due to strong nonlinear dependence on the DLO configuration.

\begin{figure*}
\centerline{\includegraphics[width=0.96\textwidth]{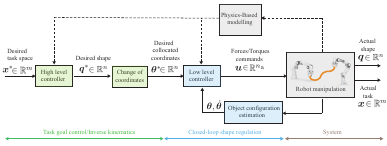}}
\caption{\small
Control architecture of the proposed approach. A high-level controller generates a desired
shape $\bm{q}^*$ from the task-space reference $\bm{x}^*$ through task-goal control or inverse
kinematics. A change-of-coordinates block maps the desired shape to collocated coordinates
$\bm{\theta}^*$. A low-level controller then computes the actuation commands
$\bm{u} \in \mathbb{R}^{n_{\mathrm{a}}}$ to regulate the object shape in closed loop. The robot manipulation
system applies the commanded forces and torques to the DLO, producing the actual object
configuration $\bm{q} \in \mathbb{R}^n$ and task-space output $\bm{x} \in \mathbb{R}^m$. An object
configuration estimator provides feedback of $\bm{\theta}$ and $\dot{\bm{\theta}}$ to the
low-level controller. A physics-based model is used to relate applied wrenches to the object
deformation and to support control and estimation.}
\label{fig:2_2}
\end{figure*}

Current DLO shape regulation methods largely emphasize geometry and often neglect dynamics, typically assuming quasi-static behavior that rarely holds in realistic manipulation scenarios \cite{navarro2013model,zhu2018dual,shetab2022rigid,jinrobust,lagneau2020automatic,shetab2023lattice}. These approaches commonly rely on locally linear deformation models via Jacobians computed online or identified offline. More advanced methods incorporate analytical DLO models to improve accuracy. For instance, geometric elastic-rod formulations have been used to plan manipulation based on static equilibrium configurations \cite{bretl2014quasi,roussel2015manipulation}. Other approaches include adaptive deformation models \cite{navarro2016automatic}, modal graph frameworks leveraging raw point cloud data \cite{yang2023modal}, finite element method (FEM)-based techniques \cite{saghour2025dual}, and extensions to constrained three-dimensional manipulation scenarios \cite{yu2025generalizable}. However, these approaches typically require slow execution to avoid inertial effects and they do not explicitly account for the object’s physical response. In contrast, dynamic control has been explored using energy-based rod models \cite{lv2022dynamic}, open-loop FEM shape control \cite{duenser2018interactive}, and closed-loop reduced-order FEM approaches \cite{koessler2021efficient}, yet these methods remain computationally heavy, mostly validated in simulation, and generally lack closed-loop stability guarantees for real-time robotic deployment.

Learning-based strategies have recently become popular for dynamic DLO manipulation, aiming to overcome quasi-static assumptions \cite{nah2020dynamic,zhang2021robots,lin2021softgym,chi2024iterative,lim_planar_2022,kuroki2023generalizable,lv2022dynamic,zhu2022challenges}. Early demonstration-based methods showed real-robot effectiveness but were limited to specific tasks and objects \cite{nair2017combining,tang2018framework}. Other approaches learn forward kinematics and dynamics offline with complex neural networks, requiring large datasets and often failing to generalize across DLOs with different physical properties \cite{yang2021learning,zhang2021deformable,yang2022learning, wang2022offline, huang2023learning}. To improve adaptability, Jacobian-based online learning has been used to estimate local linear deformation models, enabling generalization to unseen objects but remaining restricted to small deformations \cite{zhu2018dual,jinrobust,lagneau2020automatic,zhu2021vision}. More recent methods combine offline and online learning to enhance generalization and handle larger deformation regimes \cite{yu2022global,yu2022shape,wang2022offline}, while additional works explore compact state representations for robust policy learning \cite{lan2025dynamic}, one-shot learning for online parameter estimation \cite{besselaar2022one} or even learning based techniques in constrained scenarios \cite{tang2024learning}.

Model-based dynamic control of DLOs is difficult because their infinite-dimensional nature often yields underactuation even in simplified models \cite{della2023model}. Recent advances in underactuated system theory enable explicit, provably stable controllers via input decoupling of Lagrangian systems \cite{pustina2024input}, motivating the development of simplified model-based control laws for dynamic DLO manipulation. Prior work with closed-loop stability guarantees is limited to planar endpoint regulation with a single manipulator \cite{tiburzio2025controlling}. In contrast, this work proposes a 3D dynamic control framework that supports richer deformations and multiple force/torque application along the object, enabling coordinated manipulation from single- to multi-arm setups and providing a general foundation for stable control of floating-base DLOs. {In summary, this work addresses a critical gap in the literature, the absence of a general framework for 3D dynamic DLO manipulation beyond quasi-static assumptions, by introducing a formulation grounded in the physical principles governing DLOs that explicitly accounts for their physical response. Unlike existing approaches, the proposed method enables fast, dynamic manipulation in regimes with non-negligible inertial effects, while ensuring provable stability and high-precision performance. This establishes a principled foundation for reliable, real-time robotic manipulation of DLOs.

\section{Collocation of Wrenches: extending the Actuation Coordinates in $\text{SE}(3)\times\dots\times\text{SE}(3)$}\label{sec:Collocation_Wrenches}

This section addresses the following question: \textit{Does there exist a change of coordinates that decouples the actuated and unactuated dynamics of a Lagrangian system subject to multiple wrench inputs?} We make this question mathematically precise in the next two subsections and then provide a constructive answer, which will play a central role in formulating and solving deformable linear object manipulation as a three-dimensional shape regulation problem in the remainder of the paper. The result, however, is not specific to DLOs: it applies more broadly to (free-floating) mechanical systems whose control inputs are multiple three-dimensional forces and torques, such as those of spaceborne platforms, aerial or underwater vehicles, or rigid bodies interacting with the environment via several spatial contact wrenches.

\subsection{General dynamics}

The general Lagrangian form of a mechanical system with actuation non-collocated on its body coordinates is\footnote{With body coordinates, we mean here the ones that naturally arise from the modelling process - for example, joint variables in a rigid robot and strain parametrization in a deformable body.}
\begin{equation}
\label{eq::gendynamics}
\bm{M}(\bm{q})\ddot{\bm{q}} + (\bm{C}(\bm{q},\dot{\bm{q}}) + \bm{D})\dot{\bm{q}} + \bm{K}\bm{q} +  \bm{G}(\bm{q}) = \bm{A}(\bm{q})\bm{u}  \; \text{,}
\end{equation}
\noindent where $\bm{M}(\bm{q}) \in \mathbb{R}^{n \times n}$, is the generalized mass matrix, $ \bm{C}(\bm{q}, \dot{\bm{q}}) \in \mathbb{R}^{n \times n}$ is the Coriolis matrix, $\bm{D} \in \mathbb{R}^{n \times n}$ is the elastic damping matrix, $ \bm{K} \in \mathbb{R}^{n \times n}$ is the stiffness matrix, $\bm{G}(\bm{q}) \in \mathbb{R}^{n}$ is the gravitational
force. Without loss of generality, $\bm{K}$ and $\bm{D}$ are computed based on Hooke-like linear elastic and viscoelastic constitutive laws, which make them configuration-independent. The term $\bm{A}(\bm{q}) \in \mathbb{R}^{n \times n_{\mathrm{a}}}$ denotes the actuation matrix\footnote{The challenge we are discussing is relevant if this is a tall matrix, with $n_{\mathrm{a}}>n$.}, which maps control inputs $\bm{u}$ into the Lagrangian generalized forces $\bm{\tau}_{\mathrm{a}}$.

In this section, we consider mechanical systems actuated by external wrenches applied at fixed points along their structure, with posture $(\bm{p}_i(\bm{q}),\bm{R}_i(\bm{q})) \in \text{SE}(3)$. Let $\bm{W}_i = [\bm{\tau}_i^{T},\bm{f}_i^{T}]^{T} \in \mathbb{R}^6$ denote the complete wrench in the spatial frame, expressed in global coordinates, where $\bm{\tau}_i \in \mathbb{R}^3$ and $\bm{f}_i \in \mathbb{R}^3$ are the pure torque and force at that point. These are defined so that the power they produce is the external product with $(\dot{\bm{p}}_i,\bm{\omega}_i)$.
The generalized forces can be written as 
\begin{equation}
\begin{split}
    \bm{\tau}_{\mathrm{a}} &= \sum_{i=1}^{n_{\mathrm{w}}} \bm{J}^{\mathrm{s}}_{i}(\bm{q})^{T} \bm{W}_i\\
& = \sum_{i=1}^{n_{\mathrm{w}}} (\bm{J}^{\mathrm{s}}_{z_{i}})^{T}(\bm{q})\bm{\tau}_i + (\bm{J}^{\mathrm{s}}_{p_{i}}(\bm{q}))^{T}\bm{f}_i,
\end{split}
\end{equation}
where $\bm{J}^{\mathrm{s}}_{i}$ is the spatial geometric Jacobian at that point, $\bm{J}^{\mathrm{s}}_{z_{i}} \in \mathbb{R}^{3 \times n}$ and $\bm{J}^{\mathrm{s}}_{p_i} \in \mathbb{R}^{3 \times n}$ denote its angular and linear components, and $n_{\mathrm{w}}$ is the number of wrenches. The total number of actuation inputs is $n_{\mathrm{a}} = 6 n_{\mathrm{w}}$ and the input vector $\bm{u} = [\bm{W}_{1}^T, \bm{W}_{2}^T, \ldots, \bm{W}_{n_{\mathrm{w}}}^T]^T$ collects all wrench components, so that $\bm{u}_i$ is either a force or torque depending on $i$. The input matrix is then
\begin{equation}\label{eqn:MODEL1_2}
\bm{A}(\bm{q}) = \begin{bmatrix}
 \bm{J}^{\mathrm{s}}_{1}(\bm{q})^{T} &
 \bm{J}^{\mathrm{s}}_{2}(\bm{q})^{T} &
 \cdots &
 \bm{J}^{\mathrm{s}}_{n_{\mathrm{w}}}(\bm{q})^{T}
\end{bmatrix}
\end{equation}
Each column of $\bm{A}$ represents the generalized wrench distribution induced by the corresponding component of the $i$-th applied wrench. As shown in Sec.~\ref{dynamic_model}, the dynamics of deformable objects belong to this class - thus justifying the interest in this general challenge in the present work.

\subsection{Goal: decoupling through coordinates transformation}



We define a system to be \emph{input decoupable} if there exists a change of coordinates $\bm{q} \mapsto \bm{\theta}$ such that the input matrix assumes the canonical form $[\bm{I} \;\; \bm{0}]^T$. This structure allows the first $m$ components of $\bm{\theta}$ to be interpreted as \emph{collocated coordinates}, i.e., variables directly affected by control inputs, while the remaining coordinates describe internal, unactuated motions. A mechanical system admits such a representation if and only if the one-form $\bm{A}^T(\bm{q})\dot{\bm{q}}$ is integrable in time. A discussion of this condition and its implications is provided in Appendix~\ref{Sec:Collocated_form}.
The collocated form enables the formulation of \emph{collocated control problems}, for which the design and analysis of stabilizing regulators are simpler and more transparent.

\subsection{Extend actuation coordinates}

This subsection answers \textit{negatively} the question stated above by looking separately at the portions of $\bm{A}^T(\bm{q})\dot{\bm{q}}$ connected to force inputs and to the ones connected to torques, and demonstrating that, while the forces applied to the deformable object can be analytically integrable, the corresponding three-dimensional torques do not admit such a treatment. We then introduce a workaround: a method that accounts for simultaneous changes to input variables and configuration coordinates, achieving the desired beneficial decoupling. 

\subsubsection{Integrability of pure forces}

Let us focus on the $i$-th input pure force $\bm{f}_i$, i.e., the component of the input vector that is dual to the linear velocity $\dot{\bm{p}}_i$. We now show that the corresponding actuation direction satisfies the integrability condition.

Consider the time integral of the generalized velocity projected along the portion of $\bm{A}(\bm{q})$ associated with $\bm{f}_i$:
\begin{equation}
    \begin{split}
        \int_0^{t} 
        \left(
        \underbrace{
        \bm{A}(\bm{q}) 
        \left[ 0 \, \dots \, I \, \dots \, 0 \right]
        }_{\text{columns corresponding to } \bm{f}_i}
        \right)^{T}
        \dot{\bm{q}} \, \mathrm{d}t
        &= 
        \int_0^{t} 
        \bm{J}^{\mathrm{s}}_{p_{i}}(\bm{q}) 
        \dot{\bm{q}} 
        \, \mathrm{d}t \\
        &= 
        \int_0^{t} 
        \dot{\bm{p}}_i(\bm{q},\dot{\bm{q}}) 
        \, \mathrm{d}t \\
        &= 
        \bm{p}_i(\bm{q}).
    \end{split}
\end{equation}
The first equality follows from the definition of $\bm{A}(\bm{q})$, while the second exploits the kinematic identity
\(
\dot{\bm{p}}_i = \bm{J}^{\mathrm{s}}_{p_{i}}(\bm{q}) \dot{\bm{q}}.
\)
The last equality is a direct consequence of time integration. Thus, the generalized velocity associated with a pure force integrates to the position coordinate $\bm{p}_i(\bm{q})$. Consequently, actuation coordinates can be introduced for each $\bm{f}_i$, as discussed in Appendix~\ref{Sec:Collocated_form}.

\subsubsection{Integrability of pure torques}

Let us now focus on the $i$-th pure torque input $\bm{\tau}_i$, i.e., the component of the input vector that is dual to the angular velocity $\boldsymbol{\omega}_i$. We examine whether the corresponding actuation direction satisfies the integrability condition.
Consider the time integral of the generalized velocity projected along the portion of $\bm{A}(\bm{q})$ corresponding to $\bm{\tau}_i$:
\begin{equation}
    \begin{split}
        \int_0^{t} 
        \left(
        \underbrace{
        \bm{A}(\bm{q}) 
        \left[ 0 \, \dots \, I \, \dots \, 0 \right]
        }_{\text{columns corresponding to } \bm{\tau}_i}
        \right)^{T}
        \dot{\bm{q}} \, \mathrm{d}t
        &= 
        \int_0^{t} 
        \bm{J}^{\mathrm{s}}_{z_{i}}(\bm{q}) 
        \dot{\bm{q}} 
        \, \mathrm{d}t \\
        &= 
        \int_0^{t} 
        \boldsymbol{\omega}_i(\bm{q},\dot{\bm{q}}) 
        \, \mathrm{d}t .
    \end{split}
\end{equation}
The first equality follows from the definition of $\bm{A}(\bm{q})$, while the second exploits the kinematic identity
$
\boldsymbol{\omega}_i
=
\bm{J}^{\mathrm{s}}_{z_{i}}(\bm{q}) \dot{\bm{q}}.
$ Thus, the integrability condition reduces to the existence of a configuration function whose time derivative equals $\boldsymbol{\omega}_i$. This condition is not satisfied, since the integral
$\int_0^t \boldsymbol{\omega}_i \, \mathrm{d}t
$ is path-dependent. This is a classical result in Lie group theory. Indeed, examples\footnote{Consider a rigid body performing a sequence of $90^\circ$ rotations: first about one axis, then about a perpendicular axis, and finally reversing both rotations in opposite order. The accumulated angular displacement differs from zero despite the body returning to its initial configuration. More generally, this is a manifestation of the holonomy of $SO(3)$.} can be constructed such that
\begin{equation}
    \oint \boldsymbol{\omega}_i \, \mathrm{d}t \neq 0 ,
\end{equation}
where $\oint$ denotes integration along a closed trajectory in configuration space returning to the initial $\bm{q}$. If the integral were expressible as a function $\alpha_i(\bm{q})$, then necessarily
\begin{equation}
\oint \boldsymbol{\omega}_i \, \mathrm{d}t
=
\alpha_i(\bm{q}_{\mathrm{final}})
-
\alpha_i(\bm{q}_{\mathrm{initial}})
=
0 .
\end{equation}
Since this is generically non-zero, no such function exists.
Therefore, pure torques are not integrable in coordinates, and no actuation coordinate can be associated with $\bm{\tau}_i$ in the sense of Appendix~\ref{Sec:Collocated_form}.

Finally, it is useful to restate the integrability condition in local orientation coordinates, as this will be needed in the following. For any local parametrization $\bm{z}_i$ of orientation,
\begin{equation}\label{eq:P_def}
\dot{\bm{z}}_{i} 
=
\bm{P}(\bm{z}_{i}) \boldsymbol{\omega}_i \,
,
\end{equation}
where $\bm{P}(\bm{z}_{i})$ is configuration-dependent and generally not equal to the identity. This is the unique $\mathbb{R}^{3\times 3}$ matrix that maps the analytic Jacobian to the geometric one. 

\subsubsection{Actuation coordinates by simultaneous input and configuration transformation}\label{Sec:Collocated_Control_Integrable}

We finally present here one of the core theoretical results of this work, where we relax the requirement of obtaining an actuation matrix of the form $[\bm{I}\;\;\bm{0}]^T$ and instead aim for a full-rank block in place of the identity matrix.

\begin{lemma}\label{lm:weak_change}
Consider the mechanical system~\eqref{eq::gendynamics}, \eqref{eqn:MODEL1_2}.  
Let the configuration transformation
\begin{equation}\label{eq:change of coordinates}
\bm{\theta}(\bm{q})
=
\begin{bmatrix}
\bm{\theta}_{\mathrm{a},1}(\bm{q}) \\
\vdots \\
\bm{\theta}_{\mathrm{a},n_{\mathrm{w}}}(\bm{q}) \\
\bm{c}_{\mathrm{u}}(\bm{q})
\end{bmatrix}
\end{equation}
be such that
$$
\bm{\theta}_{\mathrm{a},i}(\bm{q})
=
\begin{bmatrix}
\bm{z}_i(\bm{q}) \\
\bm{p}_i(\bm{q})
\end{bmatrix}
\in \mathbb{R}^6,
\qquad
\bm{c}_{\mathrm{u}}(\bm{q}) \in \mathbb{R}^{n-n_{\mathrm{a}}},
$$
where $\bm{z}_i$ is any local parametrization of $SO(3)$ and $\bm{c}_{\mathrm{u}}$ is any function such that the Jacobian $\partial \bm{\theta}/\partial \bm{q}$ is full rank.

Then, in the coordinates $\bm{\theta}$, the input matrix becomes
\begin{equation}\label{eq:new_actuation}
\bm{A}'(\bm{\theta})
=
\begin{bmatrix}
\bm{\Psi}^{-T}(\bm{z}) \\
\bm{0}_{(n-n_{\mathrm{a}})\times n_{\mathrm{a}}}
\end{bmatrix},
\end{equation}
where
$$
\bm{\Psi}(\bm{z})
=
\mathrm{diag}\!\left(
\bm{\Psi}_1(\bm{z}_1),
\ldots,
\bm{\Psi}_{n_{\mathrm{w}}}(\bm{z}_{n_{\mathrm{w}}})
\right)
\in \mathbb{R}^{n_{\mathrm{a}}\times n_{\mathrm{a}}},
$$
with
$$
\bm{\Psi}_i(\bm{z}_i)
=
\begin{bmatrix}
\bm{P}(\bm{z}_i) & \bm{0} \\
\bm{0} & \bm{I}
\end{bmatrix}
\in \mathbb{R}^{6\times 6}.
$$
\end{lemma}

\begin{proof}

We begin by differentiating the coordinate transformation defined in~\eqref{eq:change of coordinates}. For each actuated location $i$, the time derivative of the corresponding coordinates reads
\begin{equation}\label{eq:act_coordinates}
\dot{\bm{\theta}}_{\mathrm{a},i}
=
\begin{bmatrix}
\dot{\bm{z}}_i(\bm{q}) \\[2pt]
\dot{\bm{p}}_i(\bm{q})
\end{bmatrix}.
\end{equation}

Using the kinematic relations
\begin{equation*}
\dot{\bm{z}}_i(\bm{q})
=
\bm{P}(\bm{z}_{i}) \bm{J}^{\mathrm{s}}_{z_i}(\bm{q})\, \dot{\bm{q}},
\qquad
\dot{\bm{p}}_i(\bm{q})
=
\bm{J}^{\mathrm{s}}_{p_i}(\bm{q})\, \dot{\bm{q}},
\end{equation*}
where $\bm{P}(\bm{z}_{i})$ is full rank and defined as in~\eqref{eq:P_def}, we obtain
\begin{equation*}
\dot{\bm{\theta}}_{\mathrm{a},i}
=
\begin{bmatrix}
\bm{P}(\bm{z}_{i}) \bm{J}^{\mathrm{s}}_{z_i}(\bm{q}) \\[2pt]
\bm{J}^{\mathrm{s}}_{p_i}(\bm{q})
\end{bmatrix}
\dot{\bm{q}} = \bm{\Psi}_{i}(\bm{z}_{i})\bm{J}^{\mathrm{s}}_{i}(\bm{q})\dot{\bm{q}}.
\end{equation*}
It is worth pausing to stress here that the equivalence above is the key structural step: the existence of the matrix $\bm{P}(\bm{z}_i)$—invertible for any local parametrization of $SO(3)$—allowing to connect the non integrable angular velocities $\bm\omega_{i}$ to an integrable quantity $\dot{\bm{z}}_{i}$.

Stacking all actuated coordinates and appending the complementary coordinates $\bm{c}_{\mathrm{u}}(\bm{q})$, the Jacobian of the full coordinate transformation becomes
\begin{equation*}
\bm{J}_{\mathrm{c}}(\bm{q})
=
\begin{bmatrix}
\bm{\Psi}_{1}(\bm{z}_1)\bm{J}^{\mathrm{s}}_{1}(\bm{q}) \\
\bm{\Psi}_{2}(\bm{z}_2)\bm{J}^{\mathrm{s}}_{2}(\bm{q}) \\
\vdots \\
\bm{\Psi}_{n_{\mathrm{w}}}(\bm{z}_{n_{\mathrm{w}}})\bm{J}^{\mathrm{s}}_{n_{\mathrm{w}}}(\bm{q}) \\
\nabla_{\bm q} \bm{c}_u(\bm{q})
\end{bmatrix}
=
\begin{bmatrix}
\bm{\Psi}(\bm{z}) \bm{A}^{T}(\bm{q}) \\
\nabla_{\bm q}\bm{c}_u(\bm{q})
\end{bmatrix},
\end{equation*}
where we used the definition of $\bm{A}$ given in~\eqref{eqn:MODEL1_2}.

The transformed actuation matrix follows from the standard rule $ \bm{A}'(\bm{\theta})
= \bm{J}_{\mathrm{c}^{-1}}^{T}(\bm{q})\bm{A}(\bm{q})$. By hypothesis\footnote{On passing, it is worth stressing that the existence of an $\bm{c}_{\mathrm{u}}$ such that $\partial \bm{\theta}/\partial \bm{q}$ is full rank, since $\bm{\Psi}(\bm{z})$ is full rank by construction and $\bm{A}(\bm{q})$ has full column rank; hence $\bm{\Psi}(\bm{z})\bm{A}^{T}(\bm{q})$ has full row rank.}, $\bm{J}_{\mathrm{c}}(\bm{q})$ is nonsingular, thus $\bm{J}_{\mathrm{c}^{-1}} = \bm{J}_{\mathrm{c}}^{-1}$. Thus, we can compute explicitly:
\begin{equation*}
\begin{split}
\bm A'(\bm q) &= \bm{J}^{-T}_{\mathrm{c}} \bm{A}(\bm q)\\ &= \bm{J}^{-T}_{\mathrm{c}} \left(\bm A(\bm q) \left(\bm\Psi^{T}\bm\Psi^{-T}\right) + \nabla_{\bm q}\bm c_u(\bm q)^{T} \bm 0\right) \\ &= \bm{J}^{-T}_{\mathrm{c}} \left( \begin{bmatrix} \bm A(\bm q)\bm\Psi^{T}, \nabla_{\bm q}\bm c_u(\bm q)^{T} \end{bmatrix} \begin{bmatrix} \bm\Psi^{-T} \\ \bm 0 \end{bmatrix} \right)\\ &= \bm{J}^{-T}_{\mathrm{c}} \left(\bm{J}^{T}_{\mathrm{c}} \begin{bmatrix} \bm\Psi^{-T} \\ \bm 0 \end{bmatrix}\right)\\ &= \begin{bmatrix} \bm\Psi^{-T} \\ \bm 0 \end{bmatrix}. \end{split}
\end{equation*}

This proves~\eqref{eq:new_actuation} and concludes the argument.

\end{proof}


We can therefore achieve the canonical collocated form via a simple input reparametrization, as discussed below.

\begin{corollary}\label{corollary1}
Under the assumptions of Lemma~\ref{lm:weak_change}, define the input transformation
\begin{equation}\label{eq:input_change} 
\bm{u} = \bm{\Psi}^{T}(\bm{z})\,\bm{\mu}.
\end{equation}
Then, in the coordinates $\bm{\theta}$ and inputs $\bm{\mu}$, the actuation matrix takes the canonical collocated form
\begin{equation}\label{eq:canonical_collocated}
\bm{A}''(\bm{\theta})
=
\begin{bmatrix}
\bm{I}_{n_{\mathrm{a}}} \\
\bm{0}_{(n-n_{\mathrm{a}})\times n_{\mathrm{a}}}
\end{bmatrix}.
\end{equation}
\end{corollary}

\begin{proof}
From Lemma~\ref{Sec:Collocated_Control_Integrable}, the actuation matrix in $\bm{\theta}$--coordinates is
\[
\bm{A}'(\bm{\theta})
=
\begin{bmatrix}
\bm{\Psi}^{-T}(\bm{z})\\
\bm{0}
\end{bmatrix}.
\]
Substituting $\bm{u}=\bm{\Psi}^{T}(\bm{z})\bm{\mu}$ into $\bm{A}'(\bm{\theta})\bm{u}$ yields
\[
\bm{A}'(\bm{\theta})\bm{u}
=
\begin{bmatrix}
\bm{\Psi}^{-T}(\bm{z})\\
\bm{0}
\end{bmatrix}
\bm{\Psi}^{T}(\bm{z})\bm{\mu}
=
\begin{bmatrix}
\bm{I}_{n_{\mathrm{a}}}\\
\bm{0}
\end{bmatrix}\bm{\mu},
\]
which proves~\eqref{eq:canonical_collocated}.
\end{proof}


\section{Dynamics of a DLO with non-negligible physical response}\label{dynamic_model}

\begin{figure*}[t]
\centering
\begin{minipage}[c]{0.55\textwidth}
  \centering
  \hfill \includegraphics[width=\linewidth]{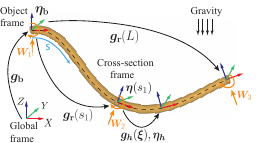}
\end{minipage}
\begin{minipage}[c]{0.4\textwidth}
  \caption{\small
Schematic representation of the kinematic model of a floating deformable linear object (DLO).
The configuration of the DLO is parameterized by the curvilinear coordinate $s$, with local
cross-section frames attached along the backbone. The global frame and the object base frame
are shown on the left. External wrenches $\bm{W}_i$ (forces and moments) are applied at discrete
grasping points along the object, while gravity acts continuously along its length. The pose
of a generic cross section is described by the homogeneous transformation
$\bm{g}(s) \in \text{SE}(3)$, and its body velocity by the twist $\bm{\eta}(s) \in \mathfrak{se}(3)$.
Relative transformations and twists over an infinitesimal segment $h$ satisfy
$\bm{g}(s+h) = \bm{g}(s)\bm{g}_h$ and $\bm{\eta}(s+h) = \bm{\eta}(s) + \bm{\eta}_h$.}
  \label{fig:model}
\end{minipage}
\end{figure*}

Here, we present a 3D dynamic model of a floating-base deformable object subject to gravity and pointwise forces and torques applied at various points along its length, simulating the wrenches that multiple robotic manipulators apply to an object in order to deform it in a specific way. While the model itself is not new, this is the first time it has been employed—and presented in this manner—for the purpose of enhancing the dynamic manipulation of deformable objects.

A floating slender soft body of length \( L \) can be modeled as a Cosserat rod, represented as a continuous stack of rigid cross-sections parameterized by a curvilinear abscissa \( s \in [0,\, 1] \), where its base is free to move in space (see Fig. \ref{fig:model}). The motion of the base configuration can be described using rigid-body motions, i.e., homogeneous transformation matrices in the Special Euclidean group $SE(3)$. The base pose is represented by the homogeneous transformation matrix \( \boldsymbol{g}_\mathrm{b} \in SE(3) \), whose generalized coordinates \( \boldsymbol{q}_\mathrm{b} \in \mathbb{R}^6 \) define the orientation and position of the object frame with respect to the global frame. Similarly, the linear and angular velocities of the moving object frame are defined by the body twist \( \boldsymbol{\eta}_\mathrm{b} \), expressed in body-frame coordinates. Using screw theory, this twist is given by $\boldsymbol{\eta}_\mathrm{b} =\boldsymbol{g}_\mathrm{b}^{-1}\,\dot{\boldsymbol{g}}_\mathrm{b}$. Next, we define the homogeneous transformation matrix between the coordinate frame attached to each cross-section and the object frame as the directed spatial curve $\boldsymbol{g}_\mathrm{r}(\bullet):s\rightarrow \boldsymbol{g}_\mathrm{r}(s) \in SE(3)$ which assigns to each point \( s \) along the rod a rigid-body configuration as follows:

\begin{equation}
    \boldsymbol{g}_\mathrm{r}(s) = \left[ \begin{array}{cc} \boldsymbol{R}(s) & \boldsymbol{r}(s) \\ \boldsymbol{0} & 1 \end{array} \right] \text{,}
    \label{eqn::homogenousTM_g}
\end{equation}

\noindent where \( \boldsymbol{r}(s) \in \mathbb{R}^3 \) is the position of the local frame, while \( \boldsymbol{R}(s) \in \mathrm{SO}(3) \) provides the orientation of the local frame. Note that, in this work, we use the subindex \( \mathrm{r} \) to denote homogeneous transformations and twists expressed with respect to the object frame. The strain field \( \boldsymbol{\xi}(s) \) and the relative velocity twist \( \boldsymbol{\eta}_\mathrm{r}(s) \) of the body with respect to the object frame are defined by the partial derivatives of Equation~(\ref{eqn::homogenousTM_g}) with respect to space \( (\cdot)' \) and time \( \dot{(\cdot)} \), respectively:

\begin{equation}\label{eqn::gprime} \boldsymbol{g}_\mathrm{r}'(s)=\boldsymbol{g}_\mathrm{r}\hat{\boldsymbol{\xi}}(s)\text{,} \hspace{0.7cm} \dot{\boldsymbol{g}}_\mathrm{r}(s)=\boldsymbol{g}_\mathrm{r}(s)\hat{\boldsymbol{\eta}}_\mathrm{r}(s) \text{,}
\end{equation}

\noindent where \( \widehat{(\bullet)} \) denotes the isomorphism between  \( \mathbb{R}^6 \) and the Lie algebra \( \mathfrak{se}(3) \), following the 
screw-theoretic representation of twists.

The relation between screw strain and velocity is established through the equality of the mixed partial derivatives in space and time:
\begin{equation}\label{eqn::vel}   \boldsymbol{\eta}_\mathrm{r}'=\dot{\boldsymbol{\xi}}-\text{ad}_{\boldsymbol{\xi}}\boldsymbol{\eta}_\mathrm{r}
\end{equation}
\noindent where \(\text{ad}_{(\boldsymbol{\bullet})}\) is the adjoint operator of \(se(3)\). Space integration of equations \eqref{eqn::gprime} (first part) and \eqref{eqn::vel} provides the homogeneous transformation and the velocity twist of any point of the object with respect to the spatial frame:

\vspace{-0.5cm}

\begin{equation}\label{eq::int_kinem_pos}
\bm{g}(s) = \boldsymbol{g_\mathrm{b}} \cdot (e^{\int_0^{s}  \widehat{\bm{\xi}} d \gamma}); \ \ \boldsymbol{\eta}(s) = \boldsymbol{\eta_\mathrm{b}} +   \text{Ad}_{\boldsymbol{g}^{-1}}\int_0^{s} \text{Ad}_{\boldsymbol{g}} \dot{\boldsymbol{\xi}} d \gamma
\end{equation}


\noindent where \(\text{Ad}_{\boldsymbol{g}}\) is the adjoint representation of \(\boldsymbol{g}\). Up to this point, the formulation remains general, and equations (\ref{eqn::homogenousTM_g})–(\ref{eq::int_kinem_pos}) define the PDE system describing the Cosserat rod. We propose employing a ''Ritz–Galerkin'' method in which the configuration is parameterized by the strain field, discretized using a truncated functional basis of spatially dependent vectors. The components along these basis vectors define a finite set of generalized coordinates, which are governed by a system of Lagrangian ODEs in time.


\begin{equation}
    \boldsymbol{\xi}(s,t)\approx \boldsymbol{\xi}_n(s,t)
= \boldsymbol{\xi}_0(s) + \sum_{i=1}^{n} q_{\mathrm{o},i}(t)\,\boldsymbol{\phi}_i(s).
    \label{eqn::discretizedStrain}
\end{equation}

\noindent where $\boldsymbol{\xi}_n \in \mathbb{R}^{6}$ is the approximated strain field and encompasses the six deformation modes including bending, twisting, shear, and elongation, $\boldsymbol{\xi}_0(s)$ is the reference strain field, $\boldsymbol{\bm{\phi}}_i(s) \in \mathbb{R}^{6}$ are the basis functions, and \( \boldsymbol{q}_{\mathrm{o}} = [q_{\mathrm{o},1}, q_{\mathrm{o},2},...,q_{\mathrm{o},n}]^{T} \in \mathbb{R}^{n_{\mathrm{o}}} \) are the generalized coordinates that 
define the deformation of the object. Equation~(\ref{eqn::discretizedStrain}) can be substituted into 
Equation~(\ref{eqn::vel}), eventually leading to the definition of the geometric Jacobian of the full system.

\begin{equation}
\label{eq::JJdot}
    \boldsymbol{\eta}(s) = \boldsymbol{\eta}_\mathrm{b}(\boldsymbol{q}_\mathrm{b}) + \mathrm{Ad}_{\bm{g}}^{-1} \int_0^s \mathrm{Ad}_{\bm{g}} \bm{B}_{\xi} d \gamma \dot{\bm{q}}_{o} = \boldsymbol{J}(\boldsymbol{q},X)\dot{\boldsymbol{q}}
\end{equation}

\noindent where $\boldsymbol{q} = [\boldsymbol{q}_\mathrm{b}^{T} \ \boldsymbol{q}_{\mathrm{o}}^{T}]^{T} \in \mathbb{R}^{n}$ are the generalized coordinates of the floating DLOs ($n=6+n_{\mathrm{o}}$) and $\bm{B}_{\xi} \in \mathbb{R}^{6 \times n_{\mathrm{o}}}$ is the matrix collecting basis functions column-wise. 



\begin{remark}
The strain field~(\ref{eqn::discretizedStrain}) can be divided into multiple elements, with the deformation of each element governed by a set of generalized coordinates. This is particularly important when the external actuation is discontinuous and higher accuracy in the object configuration is required. In this work, we suggest dividing the object into the same number of elements as the intermediate external forces, since these forces introduce discontinuities in the first derivative of the strain that cannot be properly captured by polynomial functions (see \cite{mathew2025reduced}).
\end{remark}

Projecting the free dynamics of the Cosserat rod \cite{Boyer_JNLS_2017} using the geometric Jacobian through D'Alembert's principle derives the generalized dynamics of the system in the standard Lagrangian form \eqref{eq::gendynamics} and with input matrix (\ref{eqn:MODEL1_2}). The spatial Jacobian of (\ref{eqn:MODEL1_2}) and the body Jacobian used in (\ref{eq::JJdot}) are related by the adjoint transformation through $\bm{J}^{\mathrm{s}}_{i} = \mathrm{Ad}_{\bm{g}}\bm{J}_i$ where $\mathrm{Ad}_{\bm{g}}$ denotes the adjoint transformation associated with $g$.

\section{Model-based closed-loop control of DLOs}\label{sec:Model_based_Control}

The general control architecture proposed in this work is shown in Fig. \ref{fig:2_2}. In what follows, we systematically outline its main components, emphasizing the critical role of the model in the control design process. Additionally, we analyze the stability and convergence properties of a broad family of controllers applicable to the proposed framework.

\subsection{Low-Level Control}


The control strategy developed in this work aims to achieve shape regulation, i.e., to drive the system (\ref{eq::gendynamics})-(\ref{eqn:MODEL1_2}) toward a desired configuration by applying wrenches at the point where the robots grasp the DLO. The proposed control law is defined as  




\begin{equation}\label{fictitious_law} 
\bm{u} =  \bm{\Psi}^{T}(\mathbf{z}) \bm{\nu}(\bm{\theta}, \bm{\dot{\theta}})
\end{equation}

\noindent where $\bm{\theta}$ is defined by the change of coordinates (\ref{eq:change of coordinates}).





\begin{theorem}
Consider a nonlinear floating-base DLO system actuated through multiple distributed wrenches (\ref{eq::gendynamics})-(\ref{eqn:MODEL1_2}) controlled by the law (\ref{fictitious_law}). Let the actuated coordinates be chosen according to (\ref{eq:act_coordinates}) and the unactuated coordinates $\bm{c}_{\mathrm{u}}(\bm{q}) \in \mathbb{R}^{n-n_{\mathrm{a}}}$  such that the Jacobian of the transformation is full rank. Then, the following properties hold:

\begin{itemize}
\item[i)] The DLO states remain bounded and converge asymptotically to an equilibrium configuration $(\bm{q}, \dot{\bm{q}})
= (\overline{\bm{q}}, \bm{0})$, provided that $\bm{\nu}$ is designed such that the collocated form~(\ref{eq_decoupling}) converges asymptotically to the equilibrium point
$(\bm{\theta}_\mathrm{a}, \bm{\theta}_\mathrm{u}, \dot{\bm{\theta}}_\mathrm{a}, \dot{\bm{\theta}}_\mathrm{u})
= (\bm{\theta}_\mathrm{a}^{*}, \overline{\bm{\theta}}_\mathrm{u}, \bm{0}, \bm{0})$ where $\overline{\bm{\theta}}_\mathrm{u}$ is a solution of
\begin{equation}\label{eq_equation_}
\bm{G}_\mathrm{u} + \bm{K}_\mathrm{ua}\bm{\theta}_\mathrm{a} + \bm{K}_\mathrm{uu}\bm{\theta}_\mathrm{u} = \bm{0}.
\end{equation}
\item[ii)] If, in addition, the condition 
\begin{equation}\label{eq_equation_12}
\bm{K}_\mathrm{uu} \; \succ\;
\frac{\partial^{2} U(\bm{\theta}_\mathrm{a}^{*}, \overline{\bm{\theta}}_\mathrm{u})}
{\partial \bm{\theta}_\mathrm{u}^{2}} .
\end{equation}

holds for all $\bm{\theta}_\mathrm{u} \in \mathbb{R}^{n-n_\mathrm{a}}$, then the solution (\ref{eq_equation_}) is unique and defines a globally asymptotically stable equilibrium of the closed-loop system.
\end{itemize}

\end{theorem}

\begin{proof}

This proof follows from the main results in Lemma~\ref{lm:weak_change} and Corollary~\ref{corollary1}. Substituting the control law~(\ref{fictitious_law}) into the model~(\ref{eq::gendynamics})–(\ref{eqn:MODEL1_2}) and applying the change of coordinates~(\ref{eq:change of coordinates}), the actuation matrix becomes $[\bm{I} \;\; \bm{0}]^T$ and the control input is $\bm{\nu}$.


Consequently, the convergence properties of the original system follow directly from the convergence guarantees of the collocated form (\ref{eq_decoupling}). Considering that the actuated variables $\bm{\theta}_\mathrm{a}$ are regulated to a constant value $\bm{\theta}_\mathrm{a}^{*}$ choosing properly $\bm{\nu}$, the presence of the unactuated dynamics implies that the system exhibits a zero dynamics of dimension $2(n-n_\mathrm{a})$. This can be easily derived by examining the residual dynamics in (\ref{eq_decoupling}) when $\dot{\bm{\theta}}_\mathrm{a}$ and $\ddot{\bm{\theta}}_\mathrm{a}$ are set to zero

\begin{equation}\label{zero_dynamics}
\begin{split}
 \bm{M}_\mathrm{uu}(\bm{\theta}_\mathrm{a}^{*},\bm{\theta}_\mathrm{u}) & \ddot{\bm{\theta}}_\mathrm{u}  +  \bm{C}_\mathrm{uu}(\bm{\theta}_\mathrm{a}^{*},\bm{\theta}_\mathrm{u},\bm{0},\dot{\bm{\theta}}_\mathrm{u}) \dot{\bm{\theta}}_\mathrm{u} + \bm{G}_\mathrm{u}(\bm{\theta}_\mathrm{a}^{*},\bm{\theta}_\mathrm{u})    \\    & +  \bm{K}_\mathrm{uu}  \bm{\theta}_\mathrm{u} + \bm{D}_\mathrm{uu}(\bm{\theta}_\mathrm{a}^{*},\bm{\theta}_\mathrm{u}) \dot{\bm{\theta}}_\mathrm{u}  = \bm{0}. 
\end{split}
\end{equation}

Consider the Lyapunov-like function

\begin{equation}
V(\bm{\theta}_\mathrm{u},\dot{\bm{\theta}}_\mathrm{u}) =
\frac{1}{2}\dot{\bm{\theta}}_\mathrm{u}^{T}\bm{M}_\mathrm{uu}(\bm{\theta}_\mathrm{a}^{*},\bm{\theta}_\mathrm{u})\dot{\bm{\theta}}_\mathrm{u}
+\frac{1}{2}\bm{\theta}_\mathrm{u}^{T}\bm{K}_\mathrm{uu}\bm{\theta}_\mathrm{u}+U_\mathrm{p}(\bm{\theta}_\mathrm{a}^{*},\bm{\theta}_\mathrm{u}).
\end{equation}

The gravitational potential $U_\mathrm{p}(\bm{\theta}_\mathrm{a}^{*},\bm{\theta}_\mathrm{u})$ is lower bounded, which implies that $V(\bm{\theta}_\mathrm{u},\dot{\bm{\theta}}_\mathrm{u})$ is also lower bounded. We then evaluate $\dot{V}$ along the trajectories of (\ref{zero_dynamics}), obtaining $\dot{V} = - \dot{\bm{\theta}}_u^{T} \bm{D}_\mathrm{uu} \dot{\bm{\theta}}_\mathrm{u} \leq 0$. Since $V$ is both radially unbounded and lower bounded, it is possible to invoke the corollary to LaSalle's invariance principle \cite{khalil2002nonlinear}, thus concluding convergence to (\ref{eq_equation_}) and the first part of the proof. 

The equilibrium (\ref{eq_equation_}) reached by the unactuated variables is not unique in general. To prove that the solution is a unique globally asymptotically stable equilibrium, let us consider the following Lyapunov-like function.

\begin{equation}
P(\bm{\theta}_\mathrm{u})
= U_{\mathrm{p}}(\bm{\theta}_\mathrm{a}^{*}, \overline{\bm{\theta}}_\mathrm{u})
+\frac{1}{2}\,\bm{\theta}_\mathrm{u}^{T}\,\bm{K}_\mathrm{uu}\,\bm{\theta}_\mathrm{u} .
\end{equation}

According to (\ref{eq_equation_}), the unactuated variables converge to a value
$\overline{\bm{\theta}}_{\mathrm{u}}$ satisfying $\bm{G}_\mathrm{u}(\bm{\theta}_\mathrm{a}^{*}, \overline{\bm{\theta}}_\mathrm{u})
+ \bm{K}_\mathrm{uu}\,\overline{\bm{\theta}}_\mathrm{u} = \bm{0}$. This condition corresponds to the gradient of $P(\bm{\theta}_\mathrm{u})$ evaluated
at the closed-loop equilibrium, implying that $\overline{\bm{\theta}}_\mathrm{u}$
is an extremum of $P(\bm{\theta}_\mathrm{u})$. Moreover, this equilibrium point is
unique, since the Hessian of $P(\bm{\theta}_\mathrm{u})$, $\partial^{2} U(\bm{\theta}_\mathrm{a}^{*}, \overline{\bm{\theta}}_\mathrm{u})/
\partial \bm{\theta}_\mathrm{u}^{2}
+ \bm{K}_\mathrm{uu}$, is positive definite by hypothesis (\ref{eq_equation_12}).
\end{proof}

\vspace{0.1cm}



Different controllers can be implemented within the control law (\ref{fictitious_law}), which ensures stability and convergence to the desired shapes of the DLO (see \cite{pustina2025analysis} for a proof of stability and convergence conditions in the collocated form). Some of the controllers used in this work are:

\vspace{0.2cm}

{\bf PD + compensation:} 

\begin{equation}\label{eq:PD_controller_comp+}
 \bm{\nu} = \bm{K}_P(\bm{\theta}_\mathrm{a}^{*} - \bm{\theta}_\mathrm{a})  - \bm{K}_D\dot{\bm{\theta}}_\mathrm{a} + \bm{G}_\mathrm{a}(\bm{\theta}_\mathrm{a}^{*}, \overline{\bm{\theta}}_\mathrm{u}) + \bm{K}_\mathrm{aa}\bm{\theta}_\mathrm{a}^{*} + \bm{K}_\mathrm{au}\overline{\bm{\theta}}_\mathrm{u}. 
\end{equation}

{\bf PD + cancellation:} 

\begin{equation}\label{eq:PD_controller+}
 \bm{\nu} = \bm{K}_P(\bm{\theta}_\mathrm{a}^{*} - \bm{\theta}_\mathrm{a})  - \bm{K}_D\dot{\bm{\theta}}_\mathrm{a} + \bm{G}_\mathrm{a}(\bm{\theta}) + \bm{K}_\mathrm{aa}\bm{\theta}_\mathrm{a} + \bm{K}_\mathrm{au}\bm{\theta}_\mathrm{u}. 
\end{equation}

\vspace{0.2cm}

{\bf P-SatI-D:}

\vspace{-0.4cm}

\begin{equation}\label{eq:PIDs_controller+}
\begin{split}
 \bm{\nu} = \bm{K}_P(\bm{\theta}_\mathrm{a}^{*} - \bm{\theta}_\mathrm{a}) & - \bm{K}_D\dot{\bm{\theta}}_\mathrm{a} + \bm{K}_I \int_{0}^{t} \bm{s}(\bm{\theta}_\mathrm{a}^{*} - \bm{\theta}_\mathrm{a}) d\rho \\ & + \bm{G}_\mathrm{a}(\bm{\theta}) + \bm{K}_\mathrm{aa}\overline{\bm{\theta}}
\end{split}
\end{equation}

\noindent where (\(\bm{s}(\bm{y}) = \tanh{(\bm{y}})\)) and $\bm{K}_P$, $\bm{K}_D$ and $\bm{K}_I$ are positive definite matrix of the control law.

\vspace{0.2cm}

{\bf Partial Feedback linearization (PFL):}

\begin{equation}\label{eq:PFL}
\begin{split}
 \bm{\nu} = (\bm{M}_\mathrm{aa}-\bm{M}_\mathrm{au}\bm{M}_\mathrm{uu}^{-1}\bm{M}_\mathrm{ua}) \overline{\bm{\nu}} -\bm{M}_\mathrm{au}\bm{M}_\mathrm{uu}^{-1} \bm{T}_\mathrm{u} + \bm{T}_\mathrm{a}
\end{split}
\end{equation}

\noindent where 

\vspace{-0.8cm}

\begin{equation}\label{eq:XX}
\begin{split}
 \overline{\bm{\nu}} = \bm{K}_P(\bm{\theta}_\mathrm{a}^{*} - \bm{\theta}_\mathrm{a}) & - \bm{K}_D\dot{\bm{\theta}}_\mathrm{a} 
\end{split}
\end{equation}

\vspace{-0.3cm}

\begin{equation}
\label{eq:XX2}
\bm{T}(\bm{\theta}) = (\bm{C}(\bm{\theta},\dot{\bm{\theta}}) + \bm{D})\dot{\bm{\theta}} + \bm{K}\bm{\theta} +  \bm{G}(\bm{\theta})  \; \text{,}
\end{equation}


 Control law (\ref{eq:PFL})-(\ref{eq:XX2}) guarantees exponential stability, with a fast convergence rate of the actuated coordinates determined by the controller design. The proof of stability for the actuated coordinates follows \cite{de2002underactuated}, while the remaining analysis is the same as in Theorem~1. We will test this controller in our simulations; however, this technique is more challenging to implement experimentally, since Coriolis and inertia effects must be canceled in the collocated part of the system.

\subsection{High-Level Control to achieve a task goal}

Given the complete model of the DLO with external wrenches applied to it developed previously, we can now define our manipulation goal as a shape regulation task. A convenient way to describe this task is to define a task space $\mathbb{R}^m$ as the forward kinematics of the DLO evaluated at a single point, or at a set of points of interest. More generally, this can be represented by an output function $\bm{h} : \mathbb{R}^n \rightarrow \mathbb{R}^m $. The problem can be formulated as follows:  More precisely, we wish to define a policy that specifies the control action $\bm{u}$ in \eqref{eq::gendynamics} such that
\begin{equation}\label{eq:control_goal}
     \bm{h}(\bm{\theta}_\mathrm{a}, \bm{\theta}_\mathrm{u}) = \bm{x}^{*}.
\end{equation}
 This task goal can, for example, be the position and orientation of one or more points along the object structure, their relative position to an external target, or directly the object shape $(\bm{\theta}_\mathrm{a}, \bm{\theta}_\mathrm{u})$ at the equilibrium. The optimization problem is therefore:

\begin{equation}\label{eqn:optimization_problem}
\begin{split}
    \min_{(\bm{\theta}, \bm{u}) \in\mathbb{F}} &
    \quad \|\bm{x}^{*}-\bm{h}(\bm{\theta}_\mathrm{a},\bm{\theta}_\mathrm{u})\|_2\\
    \mathrm{s.t.} \ & \bm{G}_\mathrm{a} + \bm{K}_\mathrm{aa}\bm{\theta}_\mathrm{a} + \bm{K}_\mathrm{au}\bm{\theta}_\mathrm{u} = \bm{u}.\\ 
    & \bm{G}_\mathrm{u} + \bm{K}_\mathrm{ua}\bm{\theta}_\mathrm{a} + \bm{K}_\mathrm{uu}\bm{\theta}_\mathrm{u} = \bm{0}.\\
    & \mathbb{F} = \{\underline{a}_{i}\leq   \bm{\theta}_{\mathrm{a},i} \leq \overline{a}_{i}, \ \ \ \ \ 1 \leq i \leq n  \}.
\end{split}
\end{equation}

\noindent where constraints on the actuated coordinates can be incorporated into the optimization framework to prevent collisions with the environment, which is particularly beneficial for manipulation tasks in constrained settings.

\begin{figure*}[t]
    \centering
    \begin{subfigure}[t]{\textwidth}
        \centering
        \includegraphics[width=\linewidth]{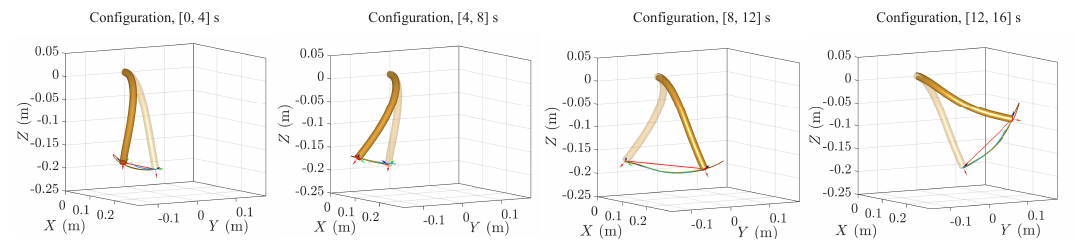}
        \caption{}
    \end{subfigure}
    \begin{subfigure}[t]{0.49\textwidth}
        \centering
        \includegraphics[width=\linewidth]{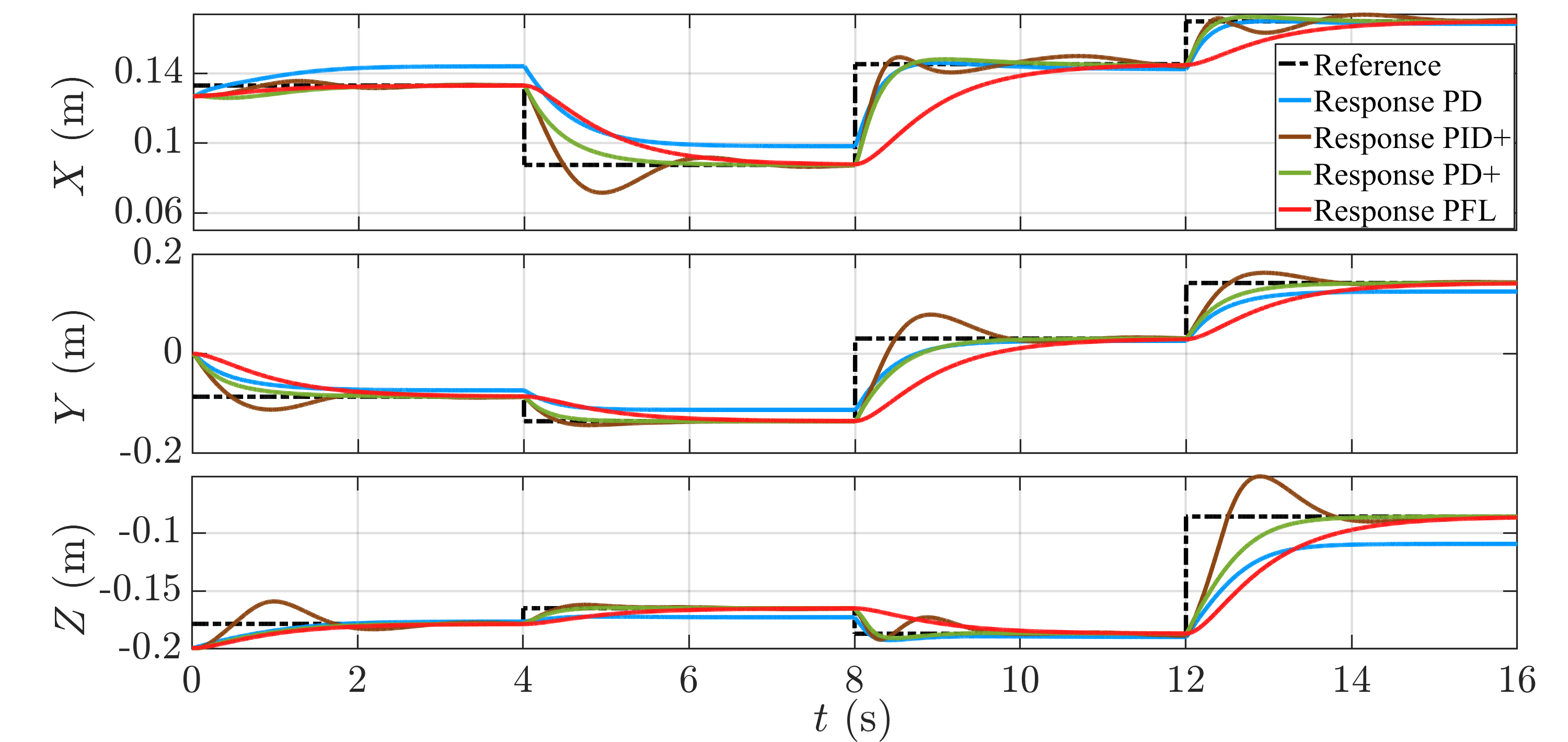}
        \caption{}
    \end{subfigure}
    \hfill
    \begin{subfigure}[t]{0.49\textwidth}
        \centering
        \includegraphics[width=\linewidth]{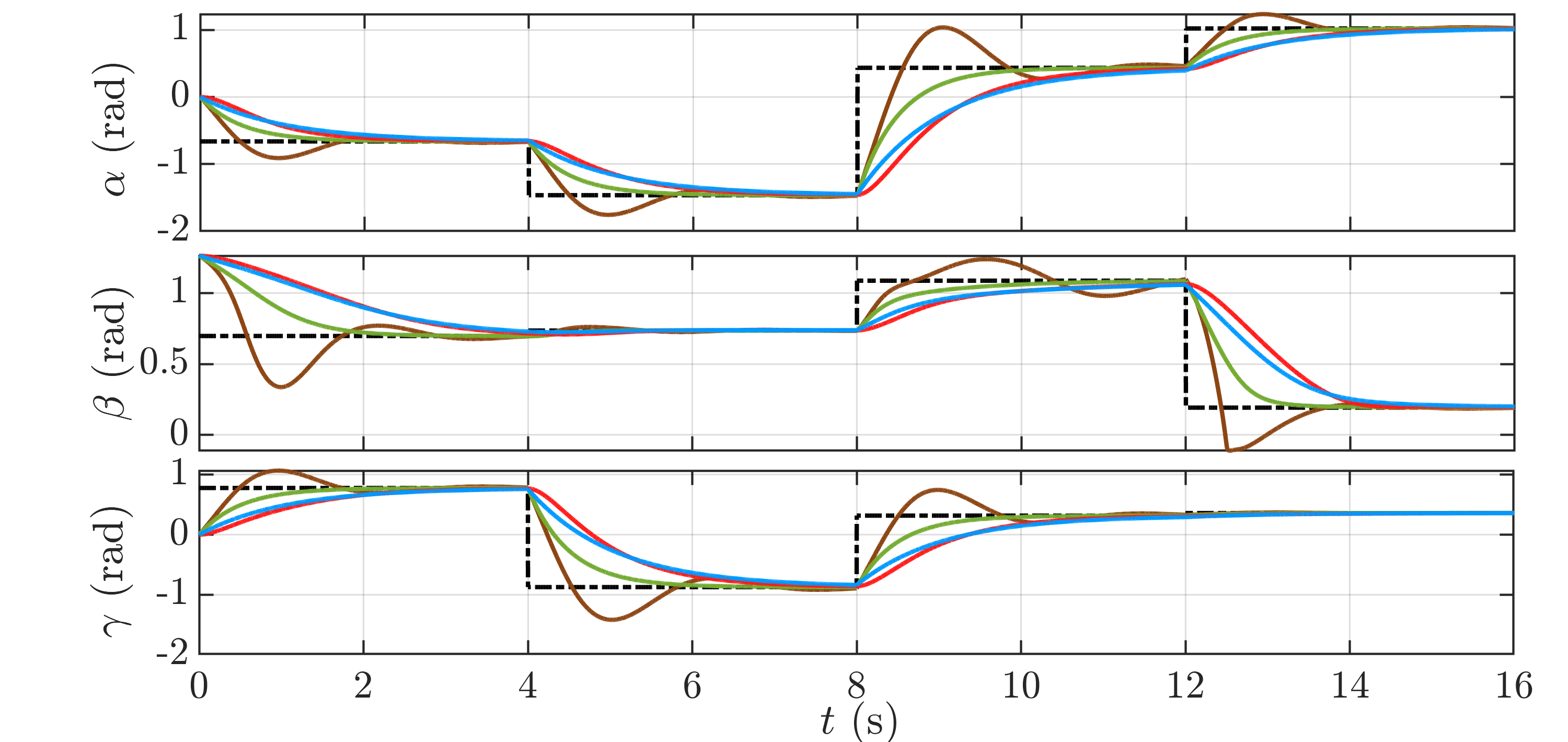}
        \caption{}
    \end{subfigure}
    \begin{subfigure}[t]{0.49\textwidth}
        \centering
        \includegraphics[width=\linewidth]{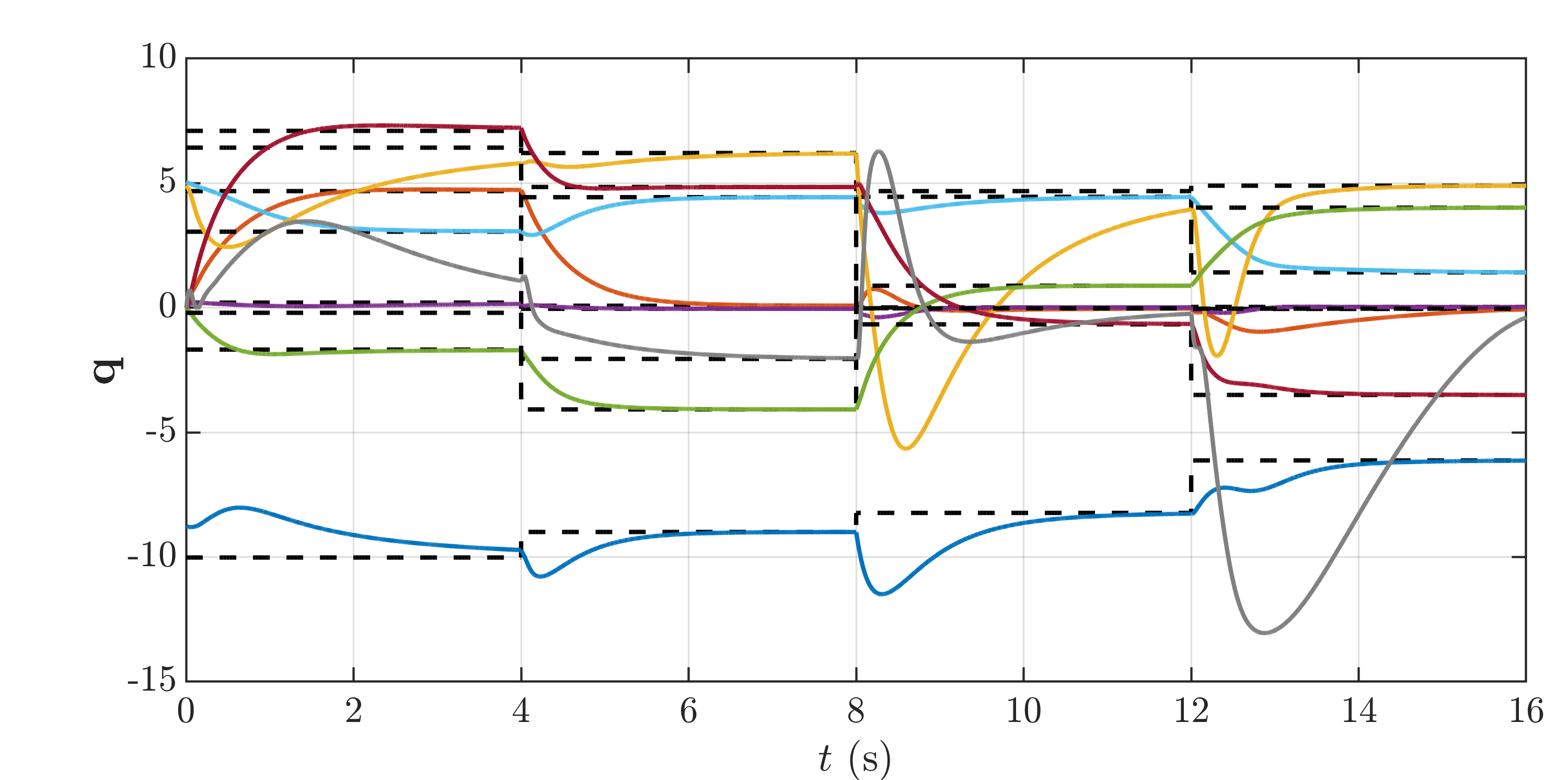}
        \caption{}
    \end{subfigure}
    \hfill
    \begin{subfigure}[t]{0.49\textwidth}
        \centering
        \includegraphics[width=\linewidth]{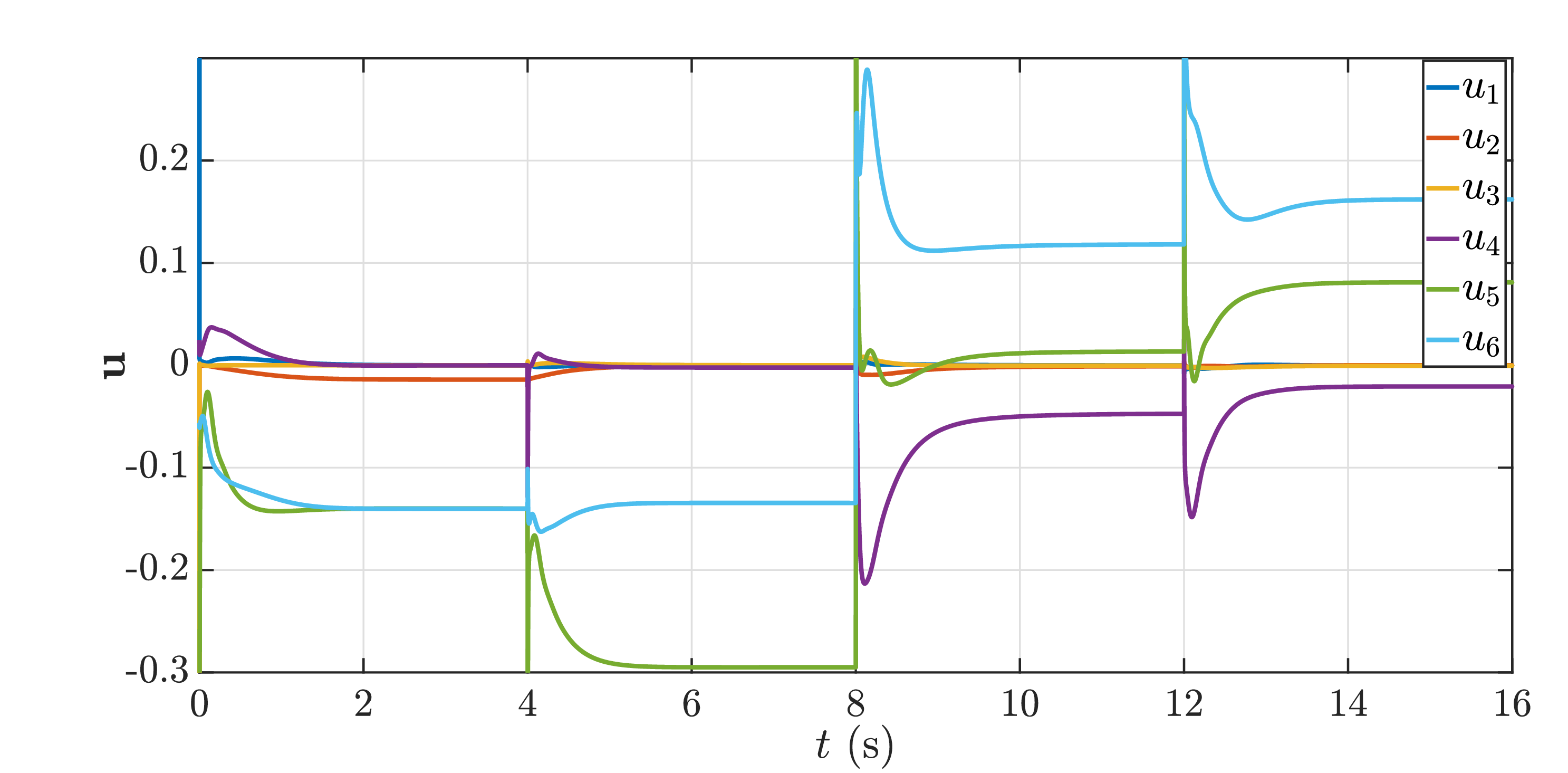}
        \caption{}
    \end{subfigure}
    \caption{\small
    Results of closed-loop control of the deformable object in simulation:
    (a) evolution of the object shape over time;
    (b)-(c) regulation of the actuation coordinates for PD controller, PID-salt controller, PD controller with compensation, and PFL;
    (c) generalized coordinates of the DLO using the PFL controller;
    (d) control actions using the PFL controller.
    }
    \label{fig:Sim_1}
\end{figure*}

\section{Simulation results}\label{sec:simulation_results}

To evaluate the effectiveness of the proposed dynamic model and control technique, we conducted numerical simulations in which several slender soft objects with different geometries were driven to achieve desired shapes under various scenarios, including cases where wrenches were applied at different points. In this section, we focus primarily on demonstrating the performance of the proposed low-level controller in regulating the shapes of DLOs, highlighting the robustness of the approach and illustrating the stability and convergence properties of the selected controllers, along with their differences and the implications of incorporating model-based components. Each system was modeled using one or multiple segments, depending on the complexity of the deformation induced by the applied wrenches. Equations (\ref{eq::int_kinem_pos})-(\ref{eq::JJdot}) can be computed recursively in both simulation and experiments by employing a quadrature approximation of the Magnus expansion to compute $\boldsymbol{g}_h$ and $\boldsymbol{\eta}_h$, as detailed in \cite{mathew2025reduced}.




\subsection{Performance Comparison of the Controllers}

In the first set of simulations, we consider a cylindrical beam with a length of $0.25,\mathrm{m}$, a radius of $3,\mathrm{cm}$, a density of $1000,\mathrm{kg/m^3}$, a Young’s modulus of $1,\mathrm{MPa}$, and a Poisson ratio of $0.5$. The system is clamped at its base, and a 6D wrench is applied at the endpoint ($n_{\mathrm{a}} = 6$). In this system, we aim to control the position and orientation of the endpoint, and we simulate the system using one element and $8$ DoF, employing a linear polynomial basis for torsion and a quadratic polynomial basis for the two bending directions ($y, z$). In addition, the base is clamped, and therefore the dimension of $\bm{q}_{b}$ is zero.

We apply the previous transformation to obtain the decoupled system, in which the actuated coordinates are the position and the Euler angles of the endpoint orientation. We perform shape regulation for four different shapes and use the controllers described in the previous section to compare their performance in simulation. The results of the experiments are shown in Fig.~\ref{fig:Sim_1}, where we plot the configuration of the object at the beginning and end of each regulation step, the desired actuated coordinates and their responses, the generalized coordinates defining the configuration, and the control inputs of the system.

Fig.~\ref{fig:Sim_2} presents the comparison among the four controllers. The results show that the conventional PD controllers exhibit a noticeable steady-state error, highlighting the importance of adding model-based components to the controllers. The PID controller eliminates the steady-state error, but introduces overshoot, which may be undesirable in certain situations. In contrast, both the PD controller with model-based terms and the partial feedback linearization achieve very good responses without steady-state error or overshoot. It is also worth noting that, in the partial feedback linearization case, the trajectory is a straight line because the dynamic effects of the actuated part are effectively canceled.  

\subsection{Evaluation of the approach under more complex conditions}

\begin{figure*}[t]
    \centering
    \begin{subfigure}[t]{0.98\textwidth}
        \centering
        \includegraphics[width=\linewidth]{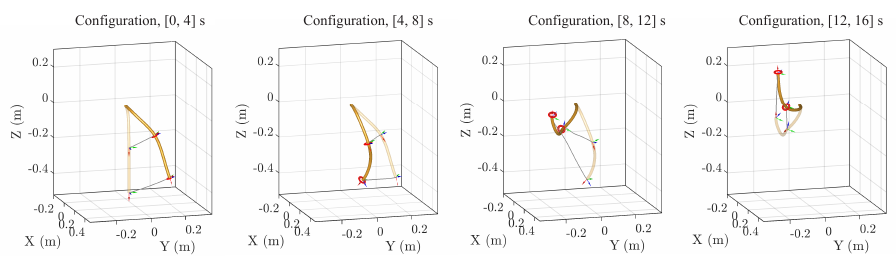}
        \caption{}
    \end{subfigure}
    \centering
    \begin{subfigure}[t]{0.98\textwidth}
        \centering
        \includegraphics[width=\linewidth]{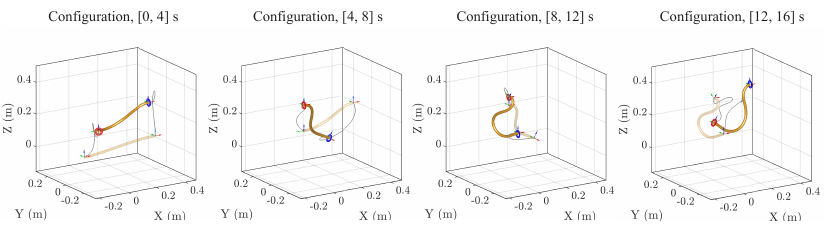}
        \caption{}
    \end{subfigure}
    \begin{subfigure}[t]{0.24\textwidth}
        \centering
        \includegraphics[width=\linewidth]{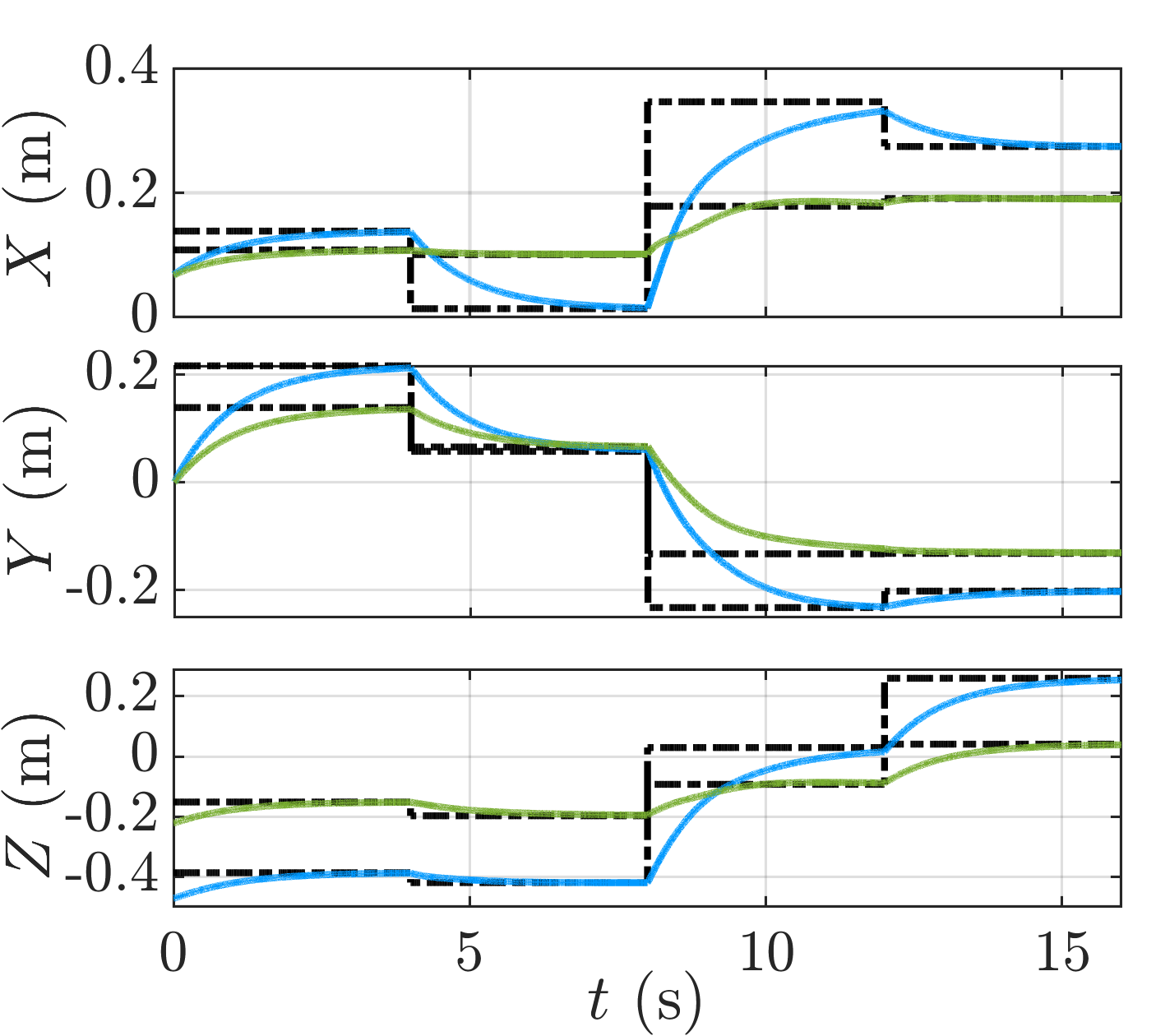}
        \caption{}
    \end{subfigure}
    \begin{subfigure}[t]{0.24\textwidth}
        \centering
        \includegraphics[width=\linewidth]{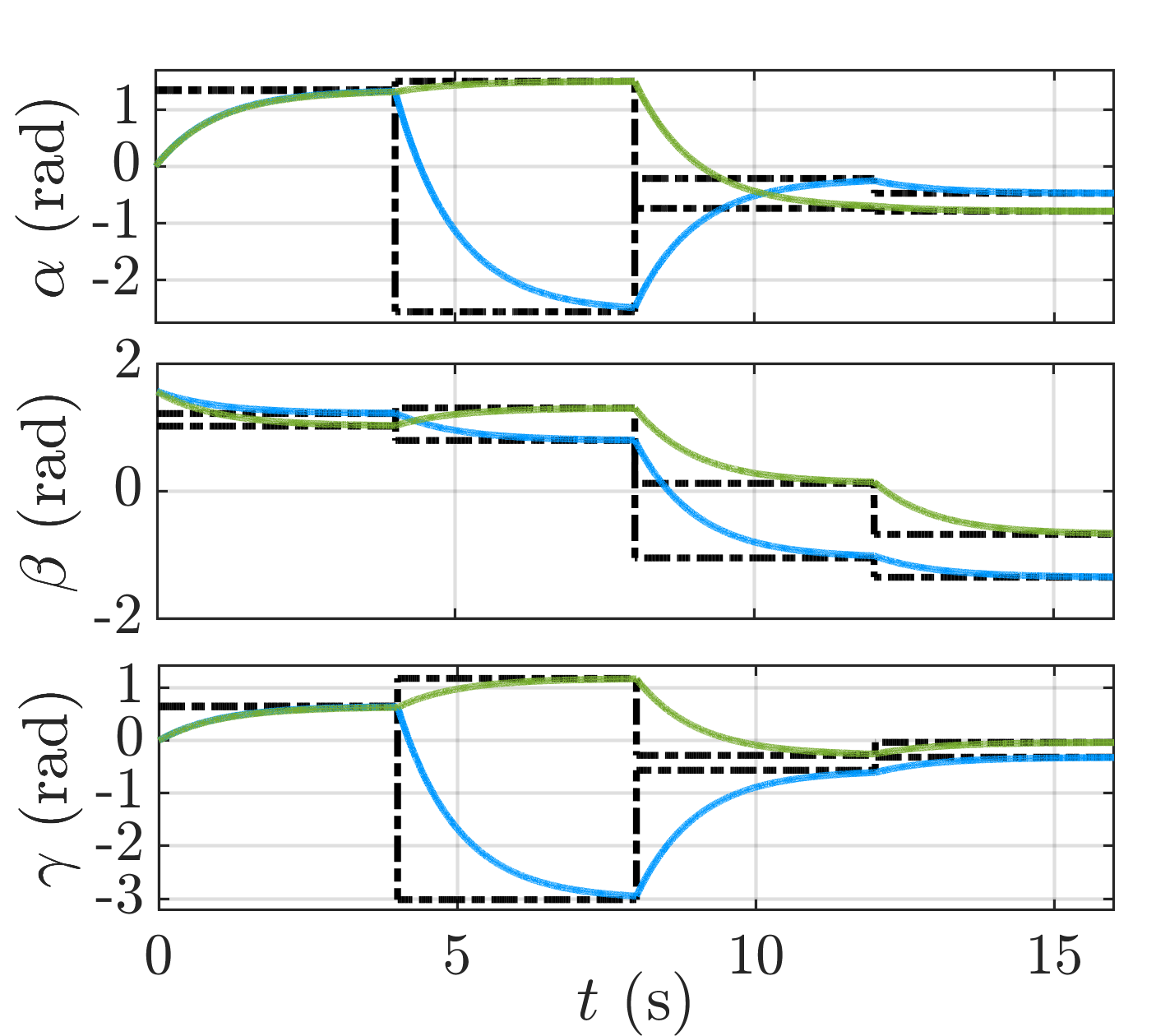}
        \caption{}
    \end{subfigure}
    \begin{subfigure}[t]{0.24\textwidth}
        \centering
        \includegraphics[width=\linewidth]{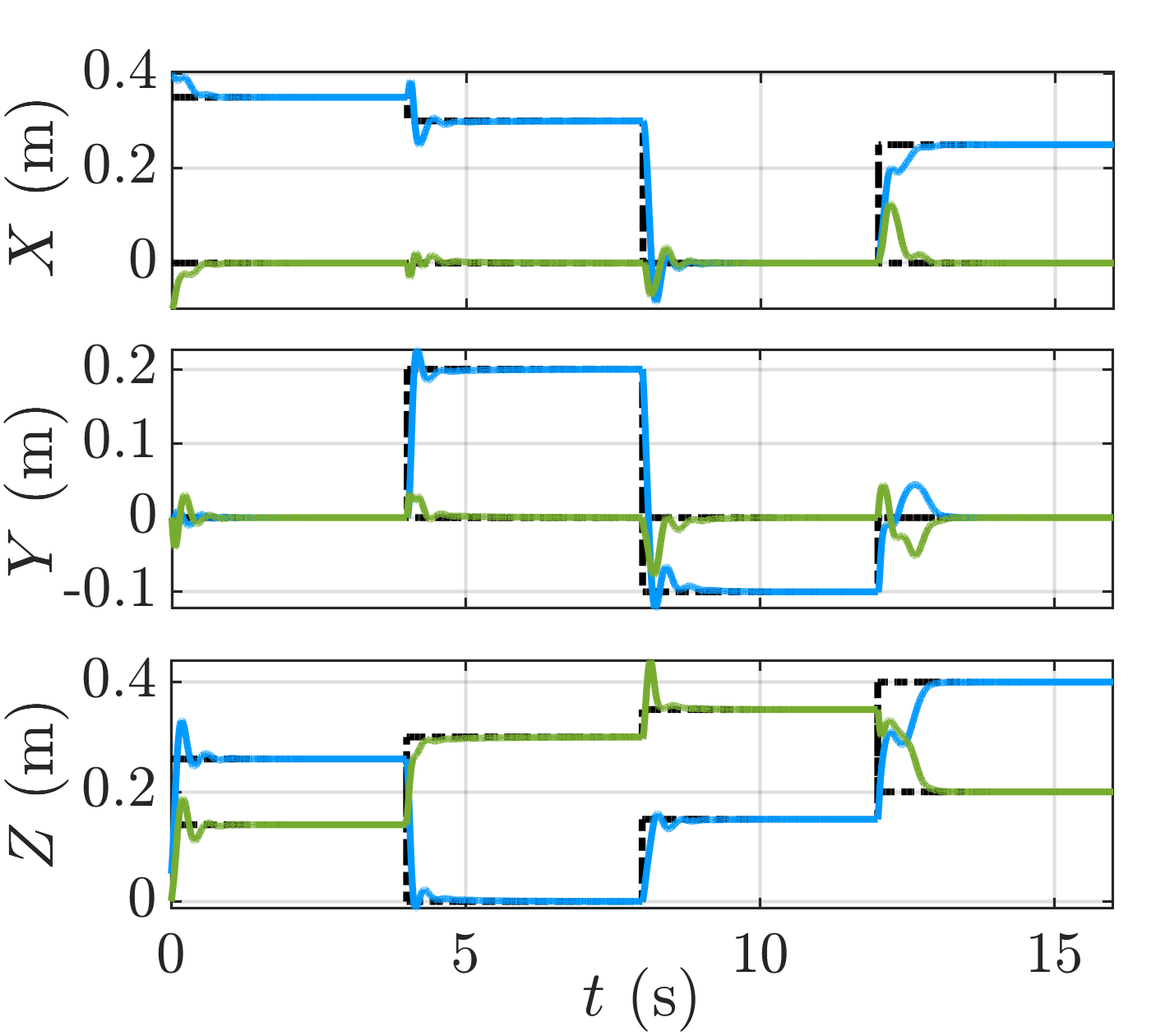}
        \caption{}
    \end{subfigure}
    \begin{subfigure}[t]{0.24\textwidth}
        \centering
        \includegraphics[width=\linewidth]{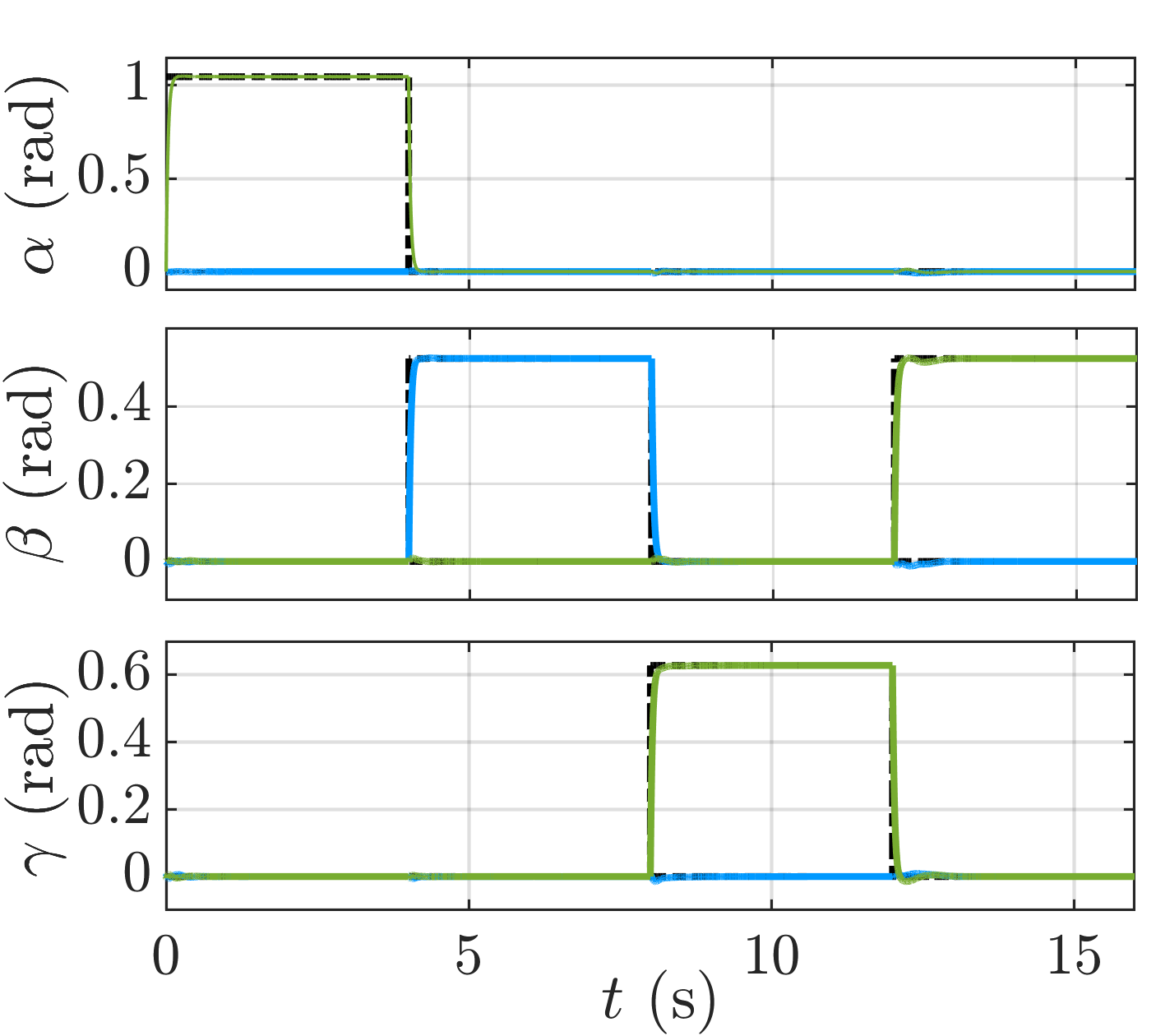}
        \caption{}
    \end{subfigure}
    \caption{\small
    Results of closed-loop control of the deformable object performed under more complex conditions: (a) shows the evolution of the object shape over time for the first case, (b) shows the evolution of the object shape over time for the second case, (c)-(d) present the regulation of the actuation coordinates for the first case: the base actuation coordinates are shown in blue, while the endpoint actuation coordinates are shown in green and (e)-(f) shows the regulation of the actuation coordinates for the second case.
    }
    \label{fig:Sim_2}
\end{figure*}

Two additional simulation scenarios are presented in this section. We consider a cylindrical beam with the same properties as before but twice the length ($0.5\,\mathrm{m}$).

The first scenario involves a beam clamped at its base, with 6D wrenches applied at the midpoint and at the endpoint ($n_a = 12$). The system actuated coordinates of the system corresponds to the positions and Euler angles at the two grasped points. We simultaneously regulate the pose at the midpoint and the endpoint of the DLO at four different target shapes. In this case, we simulate the system using two elements, since the wrench is applied at an intermediate point of the DLO and, as mentioned in the modeling remark, this choice improves the accuracy of the deformation representation. Each element is modeled with $6$ DoF, employing a quadratic polynomial basis for the two bending directions ($y$ and $z$). Again, the base is clamped; therefore the dimension of $\bm{q}_{b}$ is zero. The second scenario uses the same beam, but now it is floating, with 6D wrenches applied at both endpoints ($n_a = 12$) to simulate bimanual manipulation. The actuation and control objectives remain similar: we regulate the pose of the two endpoints using a set of four target shapes. For these two examples, we employ a PD controller with model-based terms. We simulate the system using one element with $6$ DoF and a floating base. The results of the simulations
are shown in Fig. \ref{fig:Sim_2}, where we plot the configuration of
the object at the beginning and end of each regulation step, the desired actuated coordinates and their responses. The simulations demonstrate the efficiency of the proposed technique and its ability to effectively control the position and orientation at the points where the wrenches are applied.

\section{Experimental Evaluation}\label{sec:experiments}

The methodology developed in this work computes the wrenches that the robots must apply at the points where they grasp an object in order to dynamically control the shape of the DLO and perform a specific manipulation task. 

\subsection{Implementation details}

To verify the effectiveness of our proposed method in the real world, we developed a robotic experimental setup composed of two Franka Panda robots mounted on a stationary table, real electrical cables as deformable objects, and a motion capture system. The camera system tracks markers placed along the object to estimate its shape, providing real-time state feedback for our closed-loop control strategy. The actuated coordinates are provided by the robots through their forward kinematics. The marker positions are used for shape reconstruction and for obtaining the unactuated coordinates. The actuated coordinates provided by the robots, and the wrenches computed by the control law are all defined with respect to the global frame, \{W\}; therefore, the results are reported in this frame. However, for implementation purposes, additional reference frames are introduced (see Fig. \ref{fig:coordinates_frames}). Each robot provides the position and orientation of its grasping point (\{LEE\} and \{REE\}) with respect to its own base frame (\{LB\} and \{RB\}). To be used by the controllers as actuated coordinates (see Section~\ref{Sec:Collocated_Control_Integrable}), these quantities must be expressed in the global frame, \{W\}. Moreover, the controllers compute the desired wrenches to be applied at the grasping points in (\{W\}). Since the robots execute wrenches with respect to their base frames (\{LB\} and \{RB\}), these values must be expressed in the corresponding robot base frames and then mapped to joint torques using each robot's Jacobian.
ROS manages communication among the computers controlling each robot and acquiring robot state information, the motion capture system providing marker positions, and an additional laptop running the control law and the shape estimation algorithm implemented in C++.

\begin{figure}
\centerline{\includegraphics[width=0.45\textwidth]{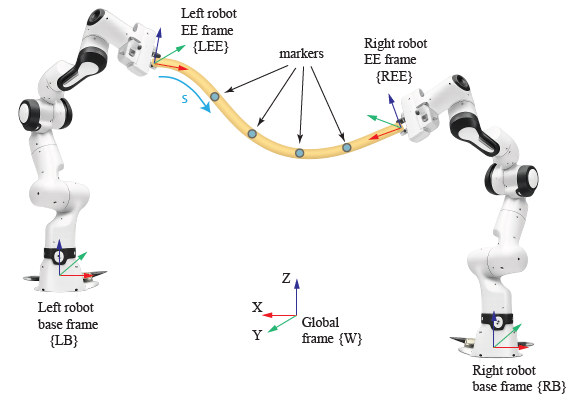}}
\caption{\small Experimental setup and coordinate frame definitions. Apart from the global frame, the base frame of each robot and the end-effector (EE) frame of each robot are also defined.}
\label{fig:coordinates_frames}
\end{figure}


\subsection{Parameter identification}\label{Sec:Parameter_identification}

Material properties for three different deformable objects were identified experimentally. To this end, the identification process was divided into several steps. First, the link properties (equivalent density and Young's modulus) were estimated by changing the orientation of the object to excite the effect of gravity on the system. The optimal parameters were then obtained by minimizing the difference between the measured marker positions and those predicted by the model at equilibrium, defined by $\bm{K}\mathbf{q} + \bm{G}(\mathbf{q}) = \bm{0}.$
The proposed identification method relies on measurements from four markers located along the rod. A similar approach was presented in \cite{feliu2024dynamic} to identify the parameters of the model. Finally, the elastic damping coefficient was estimated by exciting the object with an initial force applied at the tip, then releasing it, measuring the time response, and fitting the model to extract the damping parameter. The measured and identified parameters of the objects are shown in Table \ref{tab:para}.

\begin{table}[H]
\centering
\caption{\small Parameters of the DLO.}
\label{tab:para}
\footnotesize
\begin{tabular}{c c c c c c}
\hline
\textbf{ID} & $L (\mathrm{m})$ & $r (\mathrm{m})$ & $\rho (\mathrm{kg/m}^3)$ & $E (\mathrm{MPa})$ & $\nu (\mathrm{MPa\cdot s})$ \\
\hline
DLO1 & 0.72 & 7 $\cdot 10^{-3}$ & 3162.6 & 25.3 & 2  \\
DLO2 & 0.76 & 9 $\cdot 10^{-3}$ & 3546.6 & 22.7 & 2  \\
DLO3 & 0.89 & 8.5 $\cdot 10^{-3}$ & 4514.6 & 40.5 & 2  \\
\hline
\end{tabular}
\end{table}

\subsection{Estimation of the configuration using efficient inverse kinematics}

In this work, we propose estimating the generalized coordinates $\boldsymbol{q}^\mathrm{e}$ that define the configuration of the DLO, so that they can be used in feedback control to compensate for elasticity and gravity effects and improve manipulation performance. To do so, we use the markers' positions $\boldsymbol{p}^{\mathrm{m}}$ and minimize the discrepancy with the estimated positions $\boldsymbol{p}^\mathrm{e}$. Given the estimated generalized coordinates vector $\boldsymbol{q}^\mathrm{e}$, the estimated transformation matrix of the $k^{th}$ marker can be expressed as follows:

\begin{equation}\label{eq_est1}
    \boldsymbol{g}^\mathrm{e}_{k}(\boldsymbol{q}^\mathrm{e}) = \boldsymbol{g}(s_{k},\boldsymbol{q}^\mathrm{e}) \boldsymbol{g}_{k},
\end{equation}

\noindent where $s_{k}$ is the distance along the manipulator's centerline until the cross-section at which the $k^{th}$ marker is attached, $\boldsymbol{g}(s_{k},\boldsymbol{q}^\mathrm{e})$ is the homogeneous transformation matrix of the centerline of the object at $s_{k}$ and $\boldsymbol{g}_{k}$ is the known transformation between the centerline at $s_{k}$ and the marker position ($\boldsymbol{R}_{k} = \boldsymbol{I}$ and $\boldsymbol{r}_{k} = [0\;y_{k}\;z_k]^T$). The estimated position of the markers $\boldsymbol{p}_{k}^\mathrm{e}(\boldsymbol{q}^\mathrm{e})$, can be obtained directly from $\boldsymbol{g}_{k}^\mathrm{e} (\boldsymbol{q}^\mathrm{e})$. We can define the position error in the markers as:

\begin{equation} 
\boldsymbol{e}_{k} (\boldsymbol{q}^\mathrm{e})={\boldsymbol{p}^\mathrm{e}_{k} (\boldsymbol{q}^\mathrm{e})-\boldsymbol{p}^{\mathrm{m}}_{k}}
\end{equation}

\noindent where $\boldsymbol{e}_{k} \in \mathbb{R}^3$. To improve the accuracy of our algorithm we can introduce the position and orientation of the points where the robots grasp the object due to this information is available because of the kinematics of the robot. This error can be defined using the logarithmic map in SE(3) between the measured homogeneous transformation of the grasping points, $\boldsymbol{g}^{\mathrm{m}}_{k}$, and the estimated one as $\boldsymbol{g}^\mathrm{e}_{k}(\boldsymbol{q}^\mathrm{e})$

\begin{equation}\label{eq_est2}
   \boldsymbol{e}_{k} (\boldsymbol{q}^\mathrm{e})=\mathrm{log}((\boldsymbol{g}^\mathrm{e}_{k}(\boldsymbol{q}^\mathrm{e}))^{-1}  \boldsymbol{g}^{\mathrm{m}}_{k})
\end{equation}

\noindent in this case $\boldsymbol{e}_{k} \in \mathbb{R}^6$.
Then, we define the following minimization problem to solve including both position errors of the markers and position and orientation error of the grasping point of the robots the inverse kinematics of the system:

\begin{figure*}[h!]
    \centering
    \begin{subfigure}[t]{0.98\textwidth}
        \centering
        \includegraphics[width=\linewidth]{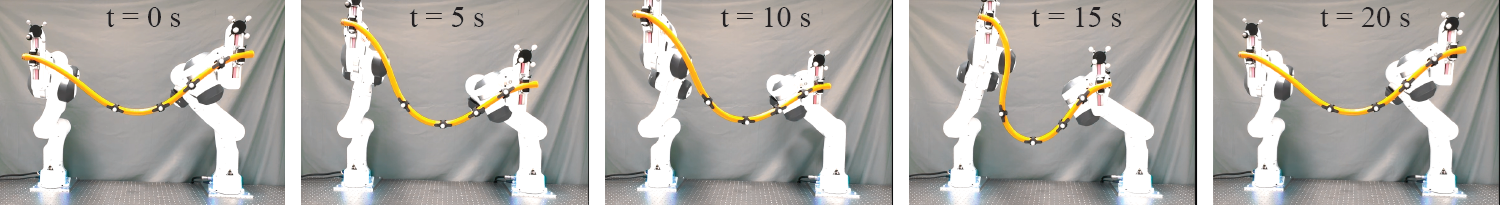}
        \caption{}
    \end{subfigure}
    \centering
    \begin{subfigure}[t]{0.98\textwidth}
        \centering
        \includegraphics[width=\linewidth]{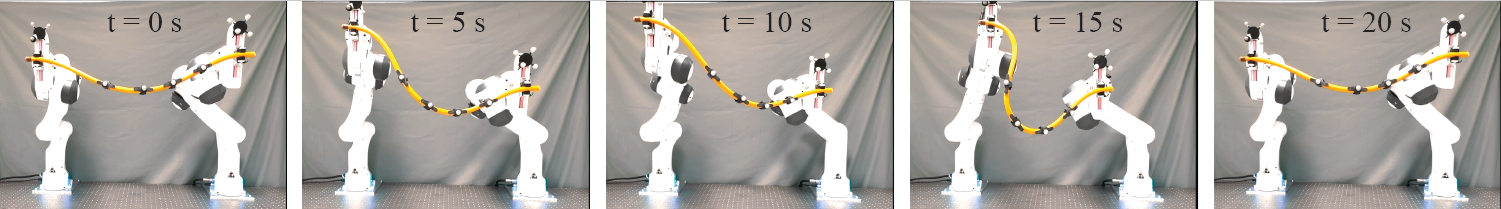}
        \caption{}
    \end{subfigure}
    \begin{subfigure}[t]{0.98\textwidth}
        \centering
        \includegraphics[width=\linewidth]{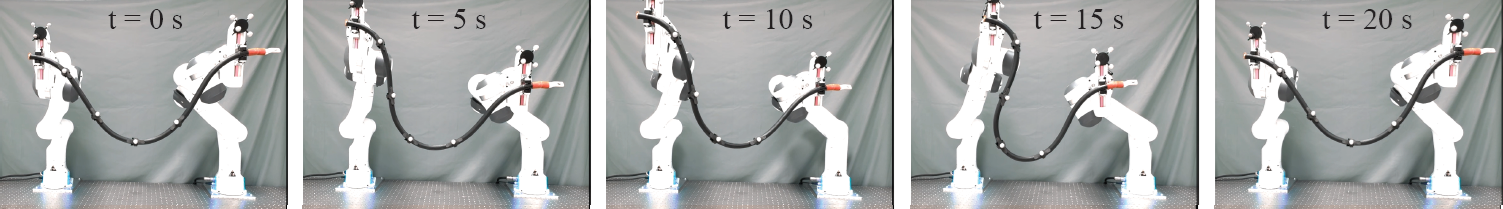}
        \caption{}
    \end{subfigure}
    \begin{subfigure}[t]{0.49\textwidth}
        \centering
        \includegraphics[width=\linewidth]{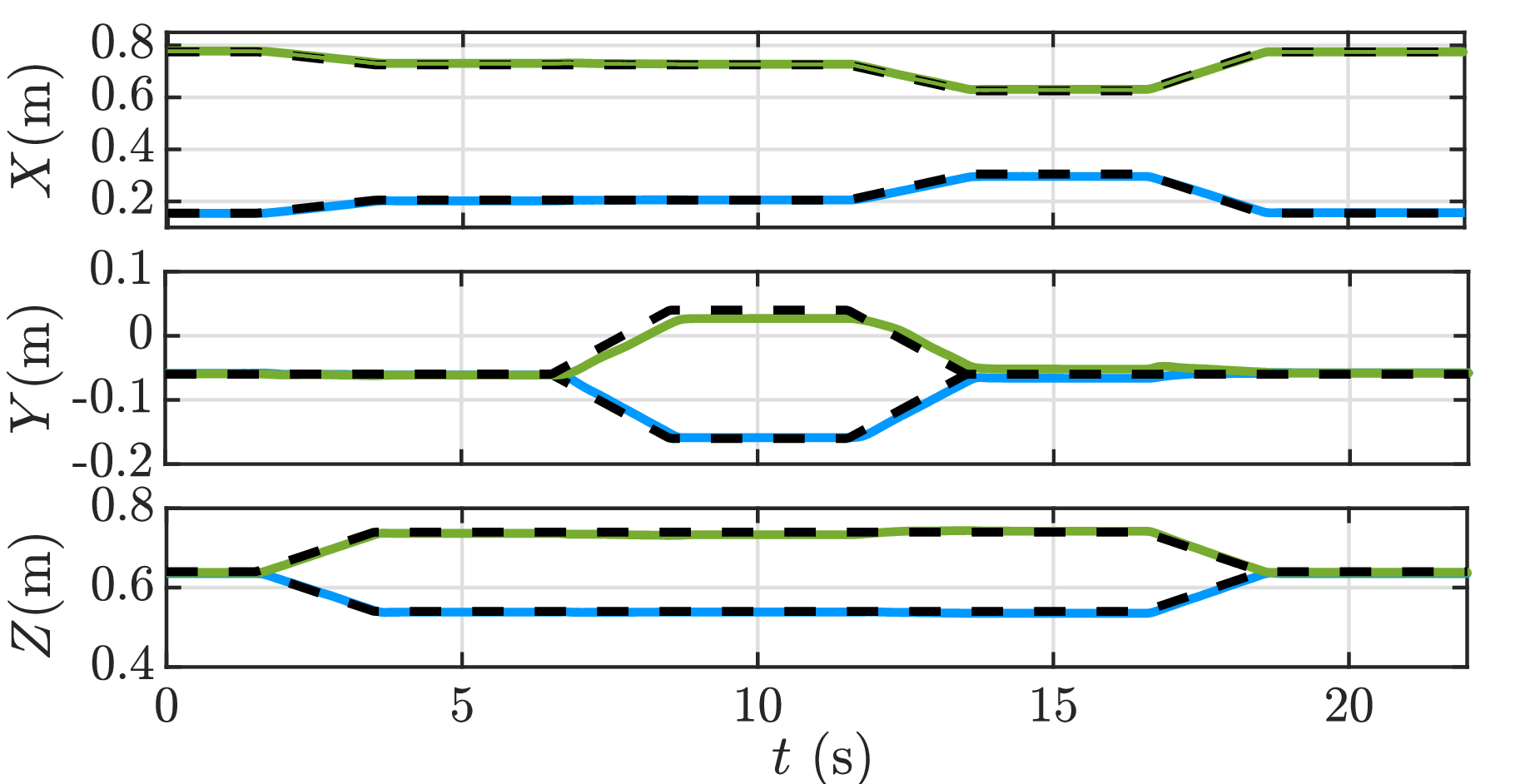}
        \caption{}
    \end{subfigure}
    \begin{subfigure}[t]{0.49\textwidth}
        \centering
        \includegraphics[width=\linewidth]{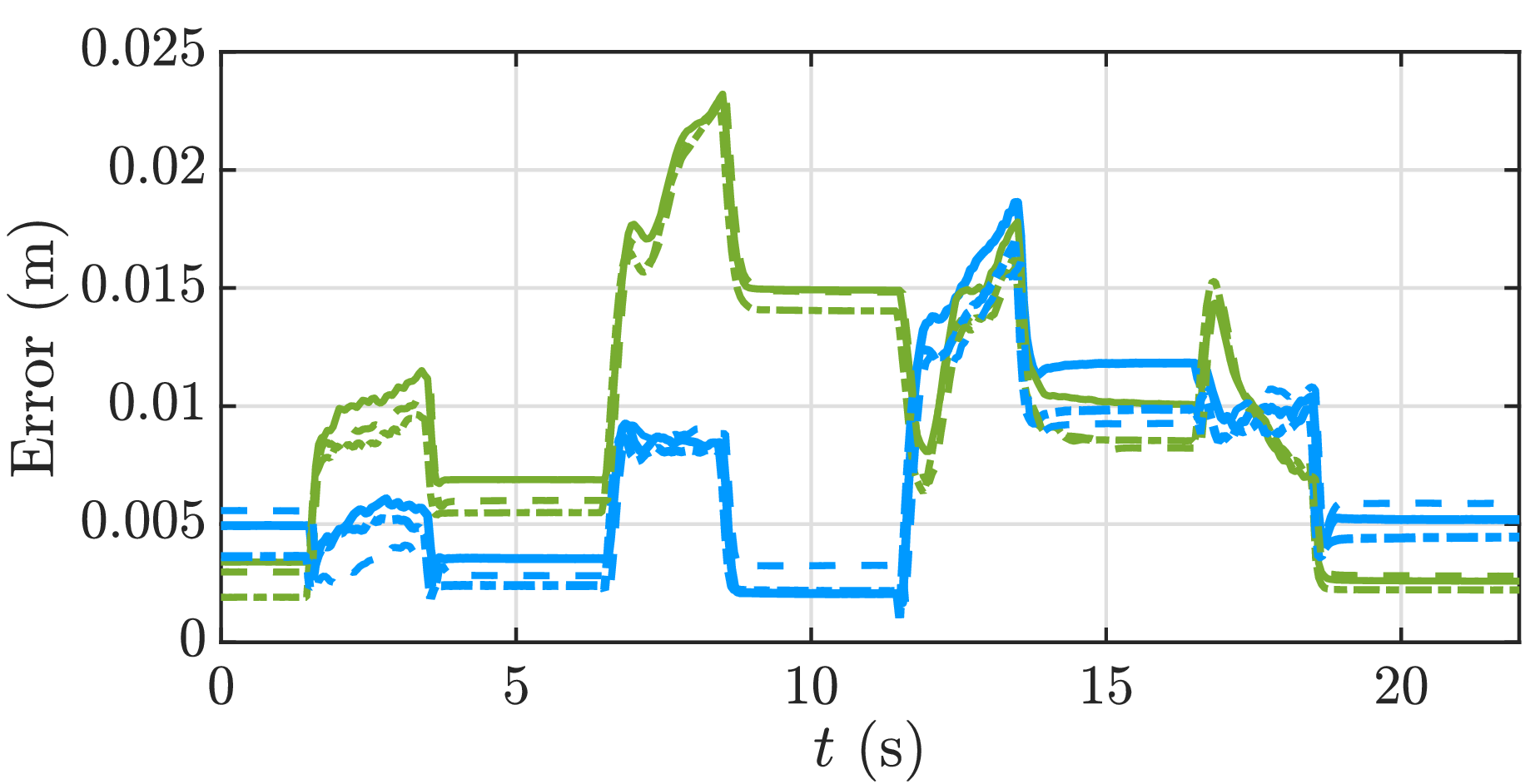}
        \caption{}
    \end{subfigure}
    \caption{\small
   Results of the experiment using multiple DLOs: (a)–(c) show the evolution of the object shape over time for DLO1–DLO3, respectively; (d) shows the evolution of the 3D positions of both end-points with respect to the desired targets (Blue line: right robot; green line: left robot); (e) shows the error in the grasping points.}
    \label{fig:Exp5}
\end{figure*}

\begin{figure*}[h!]
\centerline{\includegraphics[width=0.99\textwidth]{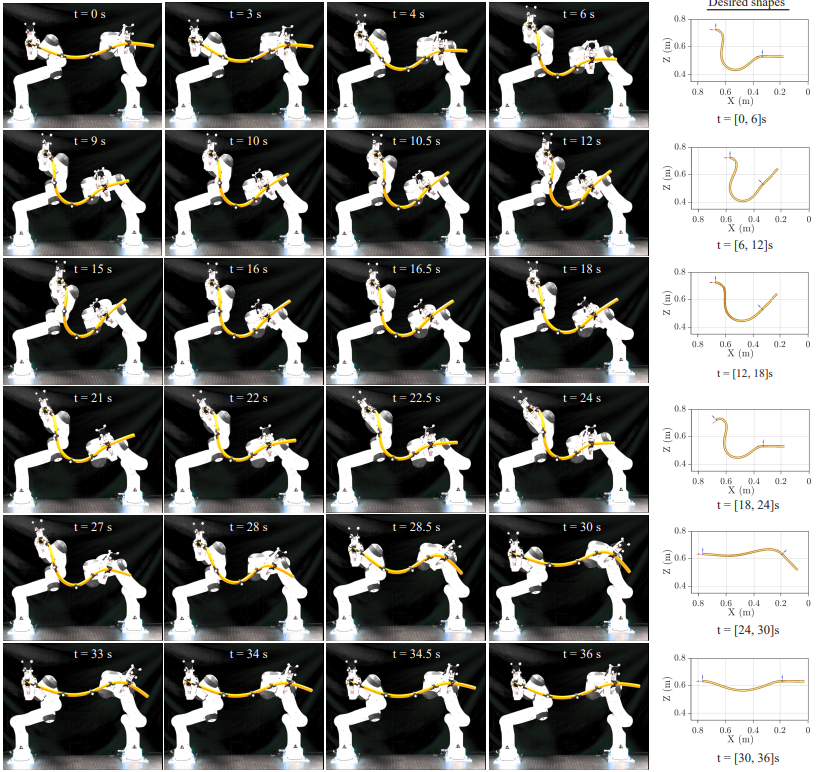}}
\caption{\small Results of the second experiment. Frame sequences of the robots’ motion together with the desired shape in the right part.}
\label{fig_Exp3}
\end{figure*}

\begin{figure*}[b!]
    \centering
    \begin{subfigure}[t]{0.345\textwidth}
        \centering
        \includegraphics[width=\linewidth]{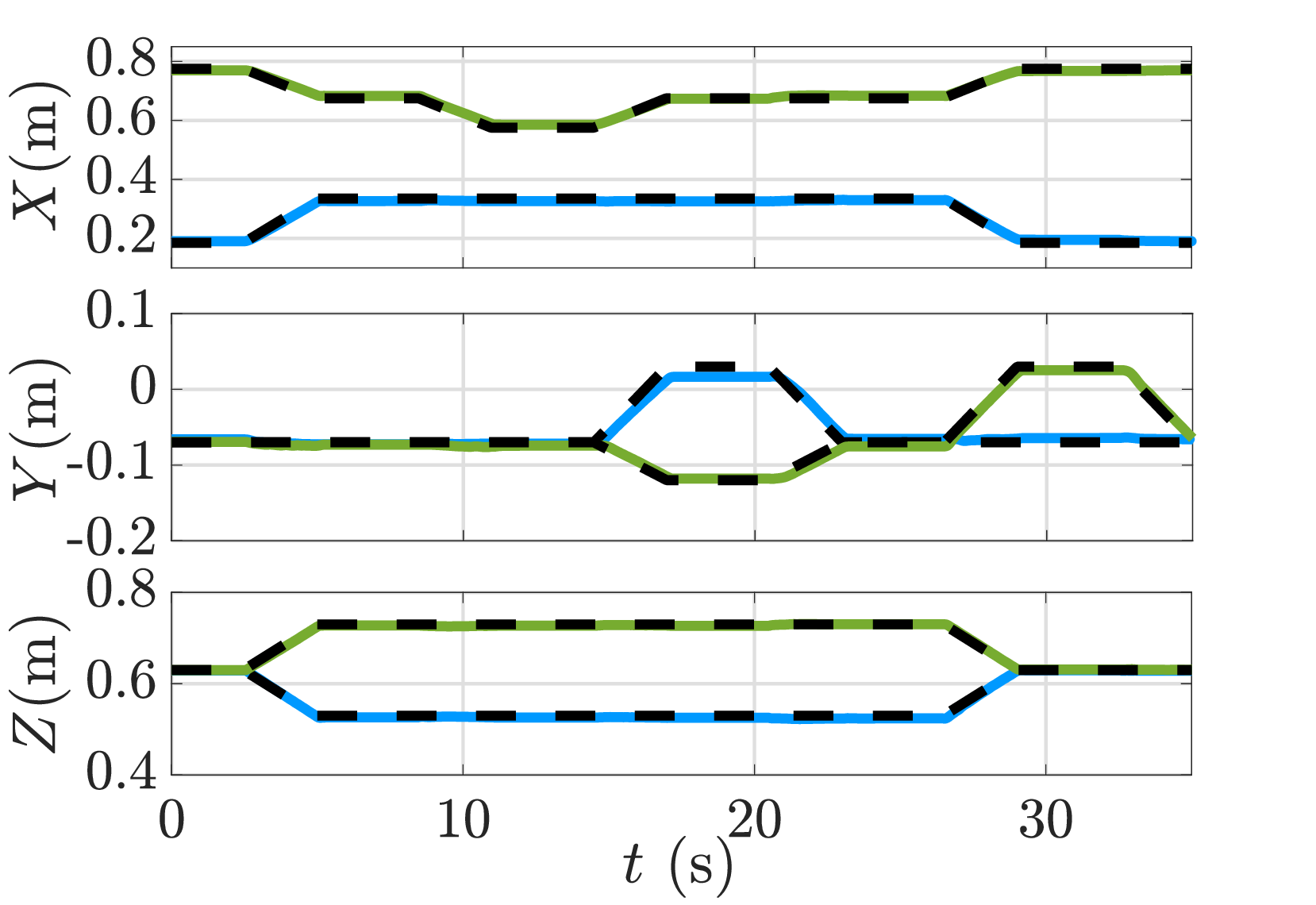}
        \caption{}
    \end{subfigure}
    \hspace{-0.5cm}
    \begin{subfigure}[t]{0.345\textwidth}
        \centering
        \includegraphics[width=\linewidth]{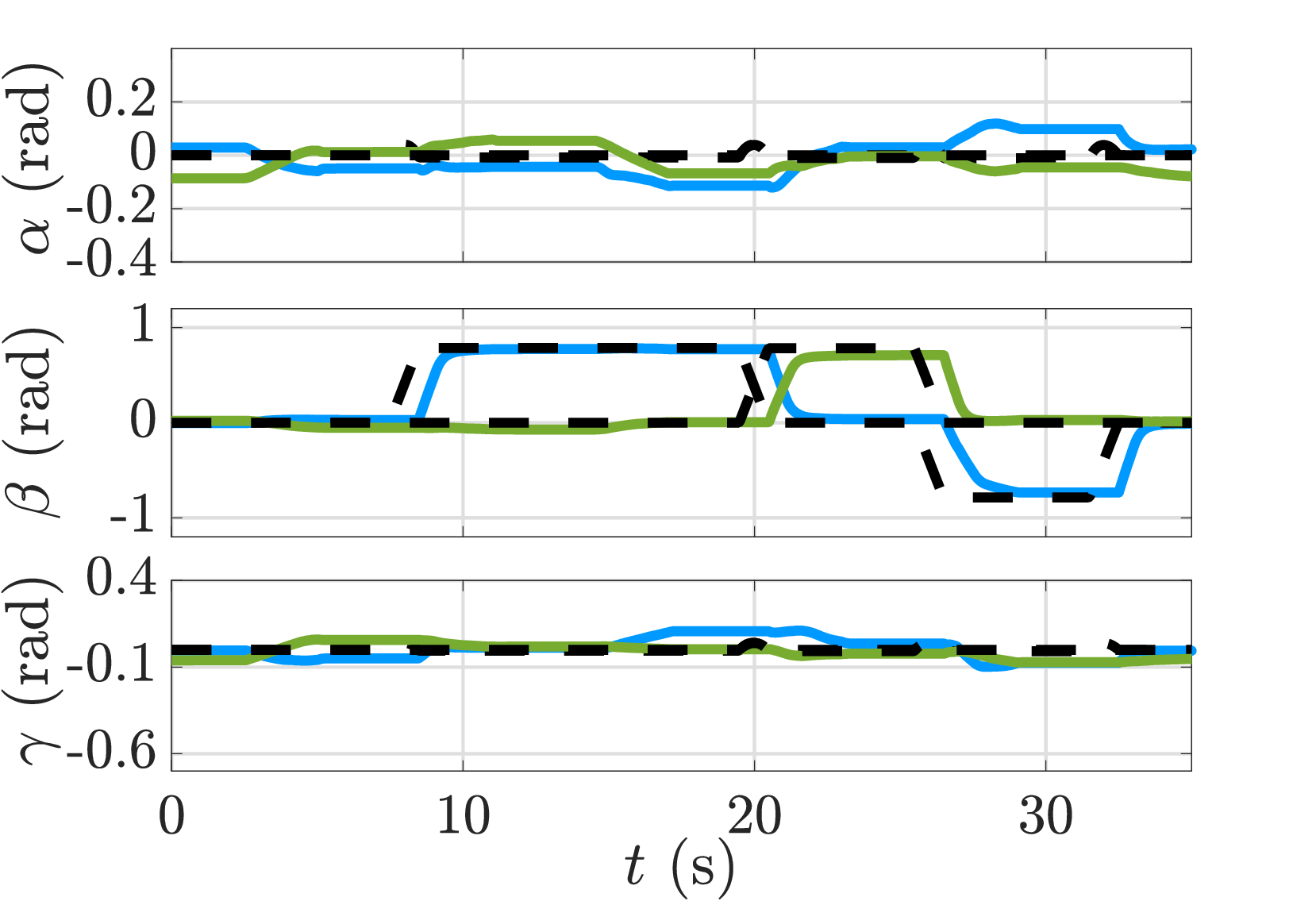}
        \caption{}
    \end{subfigure}
    \hspace{-0.7cm}
    \begin{subfigure}[t]{0.35\textwidth}
        \centering
        \includegraphics[width=\linewidth]{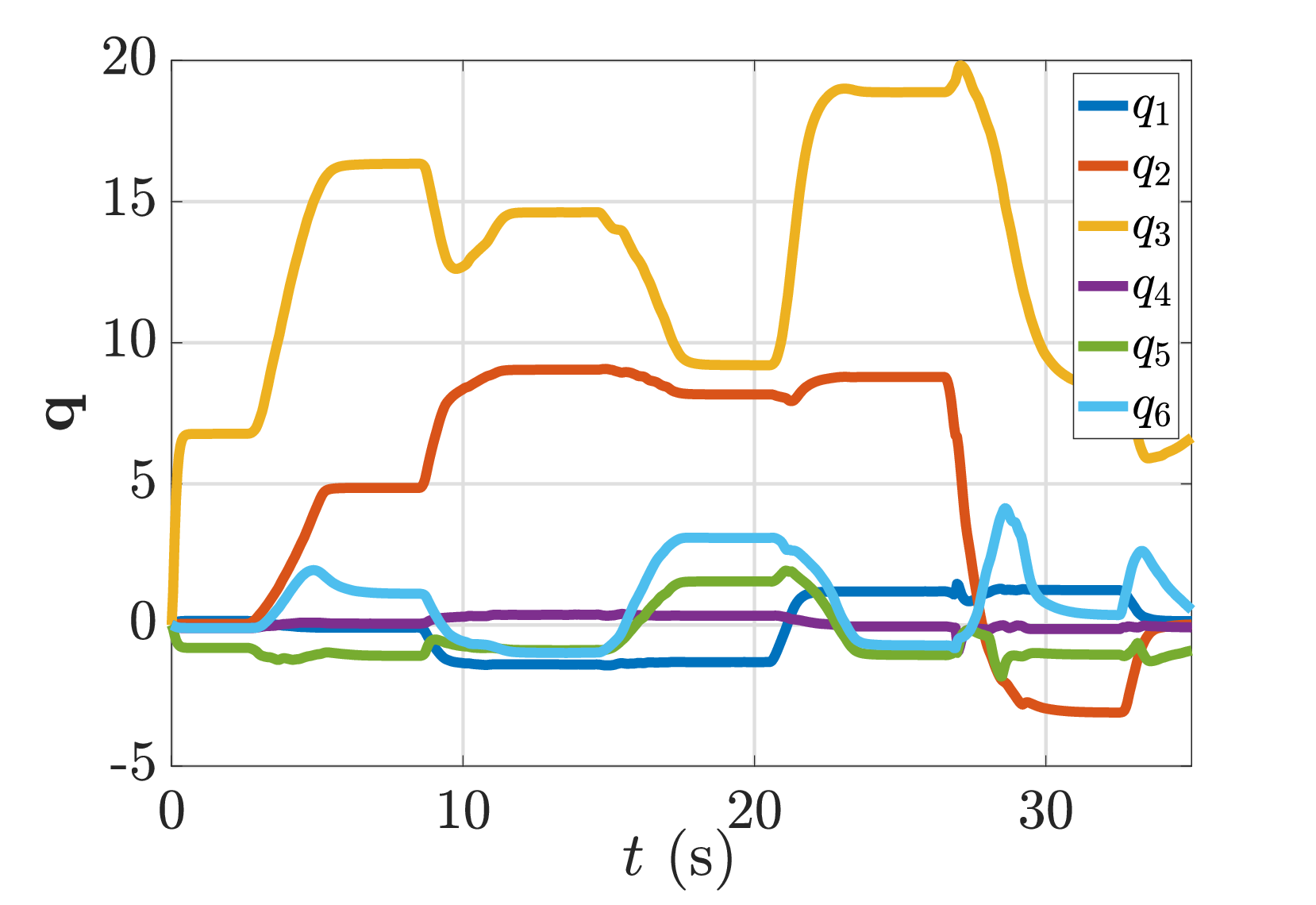}
        \caption{}
    \end{subfigure}
    \caption{\small
    \small Results of the second experiment. (a)-(b) Actuated coordinates, $\bm{\theta}_{a}$, pose of the fixed locations where the robots grasp the object (Blue line: right robot; green line: left robot.). (c) Estimation of the object configuration using the proposed estimation algorithm.
    }
    \label{fig_Exp3_2}
\end{figure*}

\begin{equation}\label{eq::ShapeEstOpt}
\begin{aligned}
\min_{\boldsymbol{q}^\mathrm{e}} \quad & \lVert \boldsymbol{e}(\boldsymbol{q}^\mathrm{e}) \lVert_{2}^{2} = \boldsymbol{f}(\boldsymbol{q}^\mathrm{e})\\
\end{aligned}
\end{equation}

\noindent where $\boldsymbol{e} = [\boldsymbol{e}_{1}^{T}, \boldsymbol{e}_{2}^{T},...,\boldsymbol{e}_{n_k}^{T}]^{T}$. We propose to solve the nonlinear least-squares problem (\ref{eq::ShapeEstOpt}) with gradient descent for implementation efficiency, both computationally and in engineering effort. This significantly reduces both computational complexity and memory usage, which is particularly relevant when the inverse kinematics must be solved repeatedly in a feedback control loop.

By the chain rule, the gradient of $\boldsymbol{f}(\boldsymbol{q}^\mathrm{e})$ is

\begin{equation}
\nabla \boldsymbol{f}(\boldsymbol{q}^\mathrm{e}) \;=\; 2 \left(\frac{\partial \boldsymbol{e}(\boldsymbol{q}^\mathrm{e})}{\partial \boldsymbol{q}^\mathrm{e}} \right) ^T \bm{e}(\boldsymbol{q}^\mathrm{e}).
\label{eq:grad}
\end{equation}

Therefore, a gradient-descent update reads

\begin{equation}
\boldsymbol{q}^\mathrm{e}_{(r+1)} \;=\; \boldsymbol{q}^\mathrm{e}_{(r)} \;-\; \alpha \,\nabla \boldsymbol{f}\!\left(\boldsymbol{q}^\mathrm{e}_{(r)}\right)
\;
\label{eq:gd_update}
\end{equation}
where $\alpha>0$ is the learning rate.

Thanks to our proposed model, we can have the analytic expression of the gradient of our minimization problem which is available from the DLO forward kinematics model, the Jacobian can be assembled by stacking the marker blocks

\begin{equation}
\frac{\partial \boldsymbol{e}(\boldsymbol{q}^\mathrm{e})}{\partial \boldsymbol{q}^\mathrm{e}} =
\begin{bmatrix}
\frac{\partial \bm{e}_1(\boldsymbol{q}^\mathrm{e})}{\partial \boldsymbol{q}^\mathrm{e}}\\[2pt]
\vdots\\[2pt]
\frac{\partial \bm{e}_{n_m}(\boldsymbol{q}^\mathrm{e})}{\partial \boldsymbol{q}^\mathrm{e}}
\end{bmatrix}
=
\begin{bmatrix}
\bm{J}^{s}_{1}\\[2pt]
\vdots\\[2pt]
\bm{J}^{s}_{n_m}
\end{bmatrix}.
\end{equation}

The iterations stop when one of the following conditions is met:

\begin{equation}
\Bigl\|\boldsymbol{f}\!\left(\boldsymbol{q}^\mathrm{e}_{(r+1)}\right)-\boldsymbol{f}\!\left(\boldsymbol{q}^\mathfrak{e}_{(r)}\right)\Bigr\|_2 \le \delta,\qquad
r \ge r_{\max}.
\end{equation}

In our experiments, we validate the real-time capabilities of our
shape estimation approach obtaining a frequency rate higher than 150 $Hz$ for our system using 6 $DoF$ showing high efficiency.

\subsection{Experimental results}

\begin{figure*}[h!]
    \centering
    \begin{subfigure}[t]{0.33\textwidth}
        \centering
        \includegraphics[width=\linewidth]{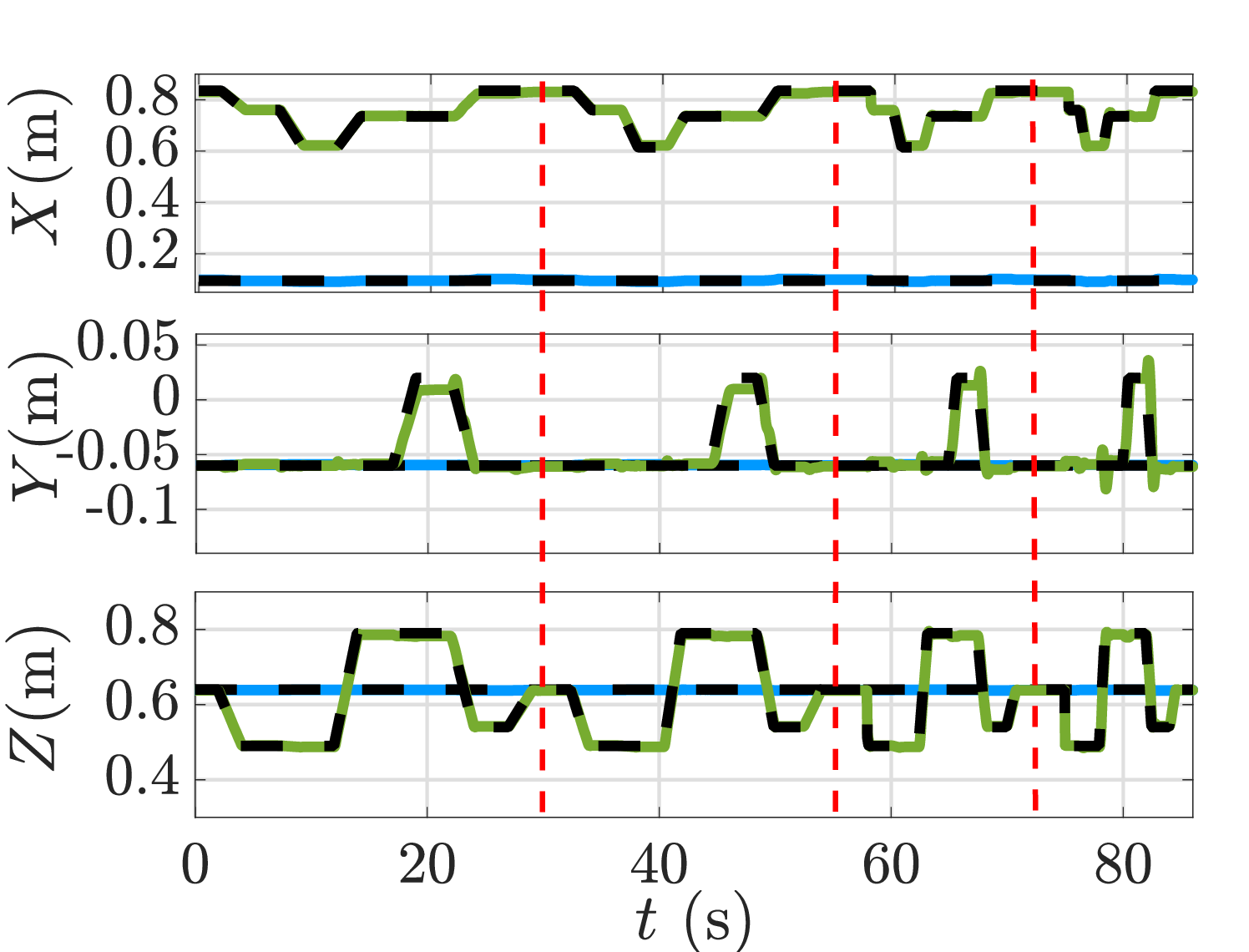}
        \caption{}
    \end{subfigure}
    \hspace{-0.1cm}
    \begin{subfigure}[t]{0.33\textwidth}
        \centering
        \includegraphics[width=\linewidth]{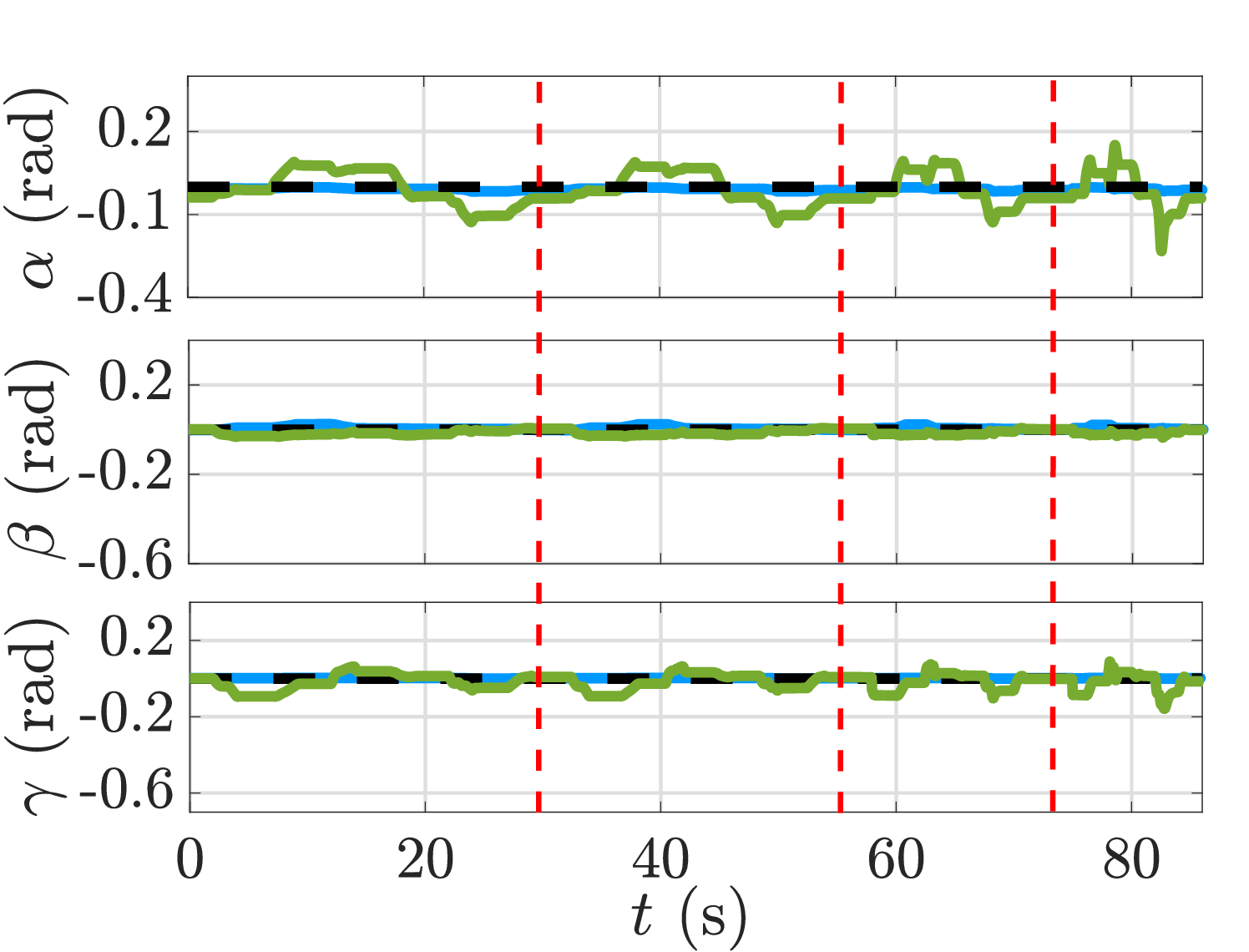}
        \caption{}
    \end{subfigure}
    \hspace{-0.2cm}
    \begin{subfigure}[t]{0.33\textwidth}
        \centering
        \includegraphics[width=\linewidth]{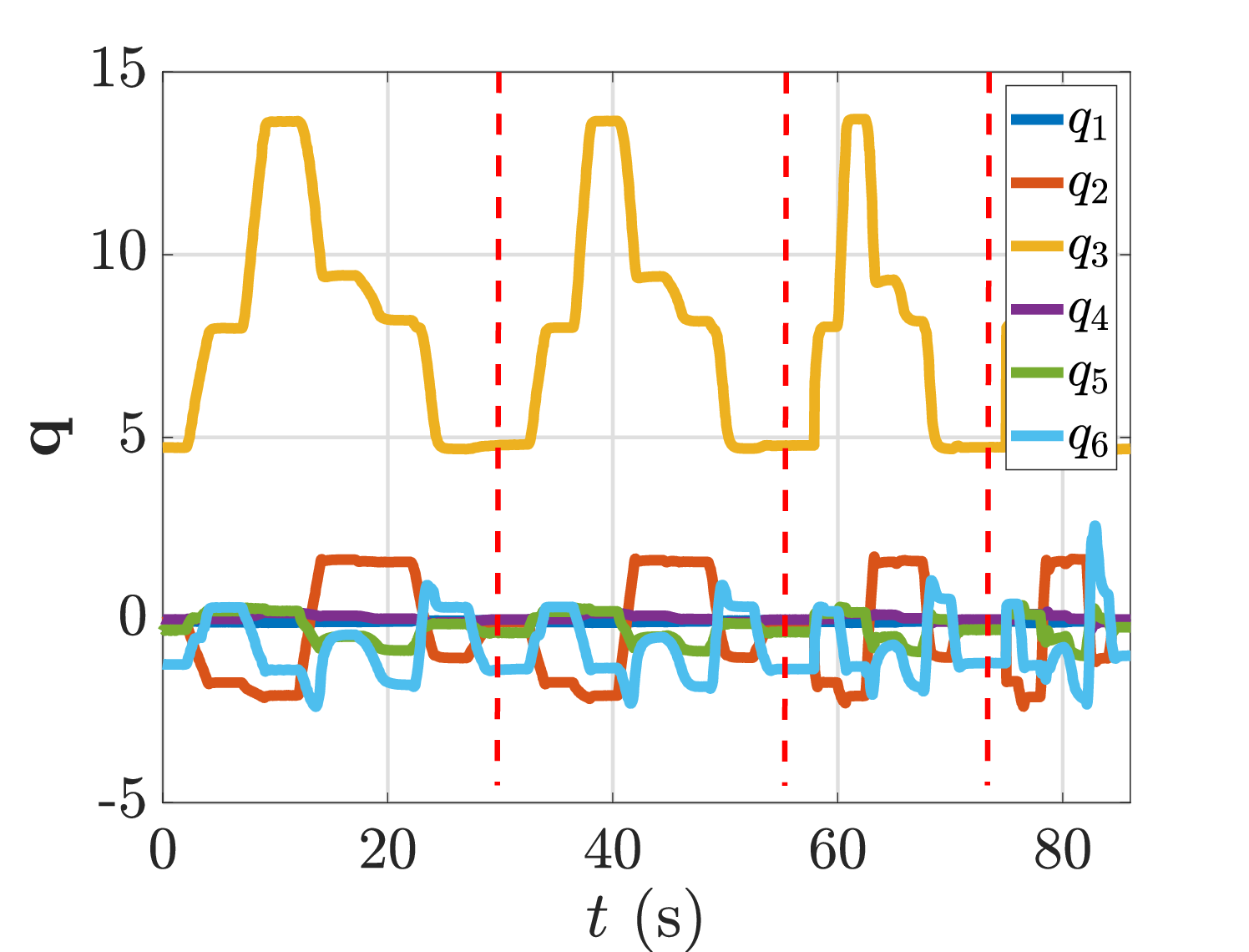}
        \caption{}
    \end{subfigure}
    \caption{\small
    \small Results of the third experiment.
(a)-(b) Actuated coordinates, $\bm{\theta}_{a}$, of the system, corresponding to the 3D position and Euler angle components of the object endpoint, while the base remains stationary.
c) Estimation of the object configuration using the proposed estimation algorithm. The system controls the endpoint position at four different target points, and the trajectory is repeated four times with progressively increasing speed. The time interval corresponding to each repetition of the trajectory is indicated by dashed lines.
    }
    \label{fig_Exp4}
\end{figure*}

In this section, we test our proposed approach in real-world scenarios across different cases, highlighting its effectiveness and comparing it with existing methods. The control framework is experimentally validated with a focus on controlling the object configuration to achieve a practical task specification. Due to the underactuated nature of DLOs, not all configurations are attainable, as they must satisfy the static equilibrium condition. Therefore, we propose solving the kinematic inversion of our model for a given desired task coordinate that satisfies the static equilibrium condition through the optimization process described in (\ref{eqn:optimization_problem}). Subsequently, we implement some of the previously introduced low-level controllers to dynamically and efficiently regulate the object’s shape in order to accomplish the desired task. In the experiments, we select polynomials with quadratic curvature in both the $y$ and $z$ directions ($6$ DoF), which correspond to the strain basis selected for the identification process and are used to evaluate the update frequency of the estimation algorithm.

\textbf{Position control of both endpoints for different DLOs:} A first set of experiments was conducted to evaluate the precision of actuated coordinate control for the  three different DLOs of Table \ref{tab:para} using a PD controller with compensation (\ref{eq:PD_controller_comp+}). The goal of this experiment is to assess the performance of the low-level controller independently of the specific properties of each DLO, thereby demonstrating the generalizability of the proposed approach across different objects. In this setup, the 3D positions of both endpoints are controlled while maintaining a constant orientation. The system inputs consist of the 3D wrenches applied at both endpoints ($n_a = 12$). Fig.~\ref{fig:Exp5} presents the results of these experiments, showing comparable accuracy across all tested DLOs. This consistent performance highlights the effectiveness of the model-based compensation term in achieving accurate regulation despite variations in object properties.

\textbf{Position and orientation control of an intermediate point and one endpoint:} A second experiment was conducted to control different shapes using the proposed approach, employing the PD controller with cancellation (\ref{eq:PD_controller+}) and the DLO2 of Table \ref{tab:para}. This experiment is intended to demonstrate the performance of the low-level controller, in which forces and torques must be applied simultaneously at different points. This experiment highlights the generality of the proposed approach by applying wrenches at different points: one robot grasps the object at one end-point, while the other robot grasps it at an intermediate point ($s = 0.6$). The desired shapes to be regulated are randomly selected by solving the static equilibrium of the system and extracting the corresponding actuated coordinates at equilibrium, which are then used as references for the low-level controller. The results of this experiment are depicted in Fig. \ref{fig_Exp3} and Fig. \ref{fig_Exp3_2}. The system inputs are again the 3D wrenches applied at both endpoints ($n_a = 12$).
We focus here on different shapes that are achieved by controlling the actuated coordinates, with particular emphasis on simultaneously controlling both position and angular coordinates of the two endpoints. As illustrated in the figure, the actuated coordinates are accurately regulated, and the desired shapes are achieved with high precision, demonstrating both the effectiveness of the low-level controller and the fidelity of the proposed model. The results also report average and maximum position errors in steady state of $0.5$ and $1$ cm, respectively, while the average and maximum errors in the orientation components are $0.04$ and $0.1$ rad, respectively. The figure also shows the estimated generalized coordinates that allow reconstruction of the object configuration, which is required to implement the controller (\ref{eq:PD_controller+}).
In this experiment, without loss of generality, the markers are positioned on the first part of the object. For the remaining part of the DLO, we assume zero deformation, as the deformation is observed to be negligible. This assumption simplifies the shape estimation problem by reducing the configuration of the DLO to six generalized coordinates.


\begin{figure*}[b!]
    \centering
    \begin{subfigure}[t]{0.49\textwidth}
        \centering
        \includegraphics[width=1\linewidth]{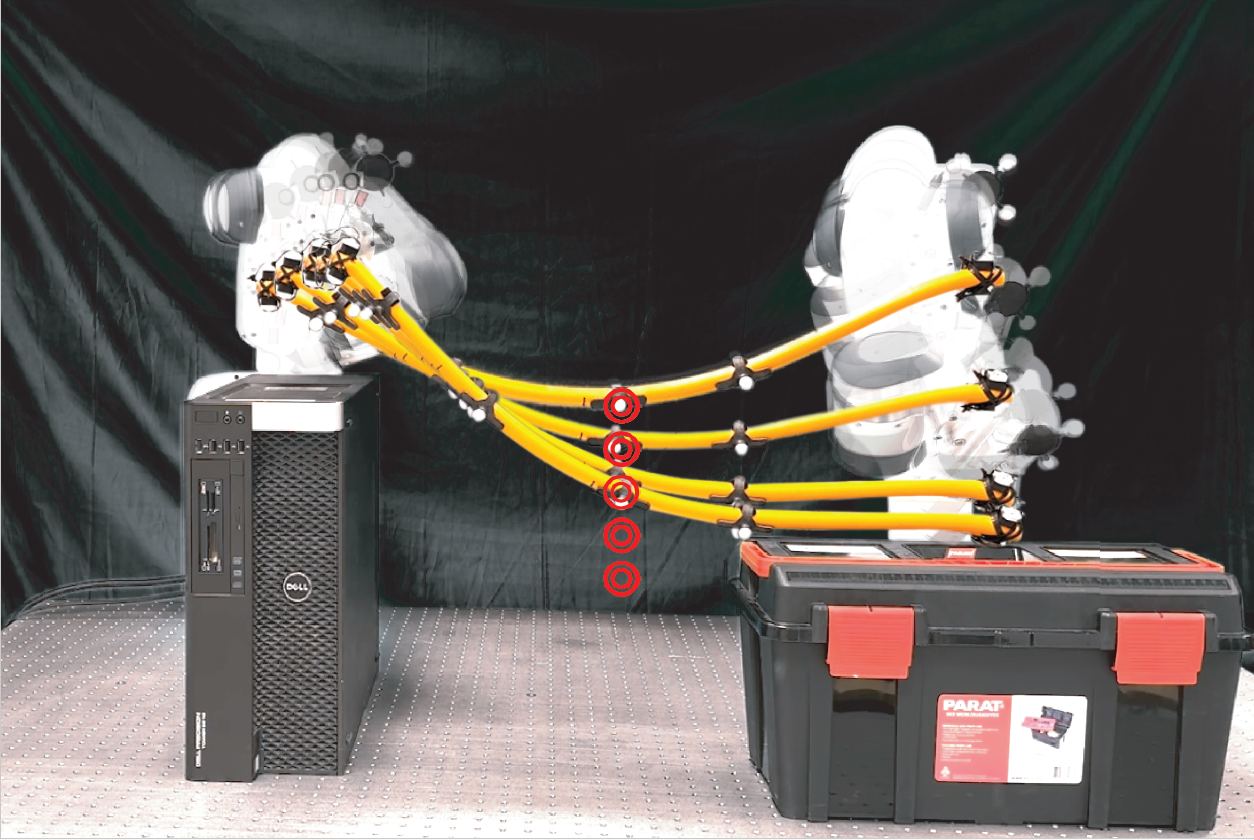}
        \caption{}
    \end{subfigure}
    \centering
    \begin{subfigure}[t]{0.489\textwidth}
        \centering
        \includegraphics[width=1\linewidth]{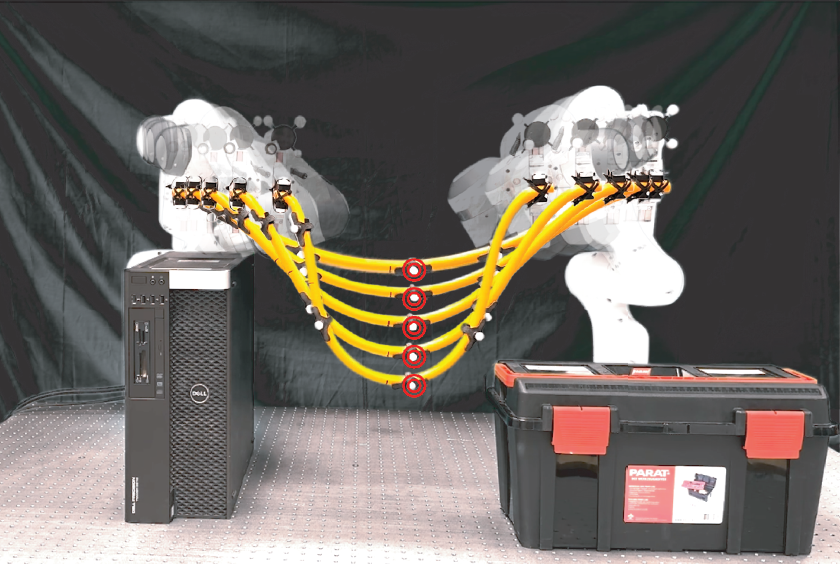}
        \caption{}
    \end{subfigure}
    \caption{\small
    Results of the experiments for the fourth example.
(a) Frame sequences of the robots’ motion using the model-free approach and (b) using the proposed approach. The desired position to be achieved by the midpoint of the DLO is marked with a red circle to emphasize the improvement in positioning achieved with our approach.
    }
    \label{fig_Exp1_base}
\end{figure*}

\textbf{Tip position–orientation control with increased speed:} we conduct a second experiment to evaluate the robustness of the proposed control approach under very fast motions using the DLO2 Table \ref{tab:para}, again employing the PD controller with cancellation (\ref{eq:PD_controller+}). The robot on the right side grasps one endpoint and keeps it stationary using its brakes, simulating a clamped boundary condition at the base of the object. Meanwhile, the other robot attempts to position the opposite endpoint at different target locations while maintaining a constant orientation and executing the motions at different speeds. The system inputs in this example are only the 6D wrenches applied at the endpoint manipulated by the left robot ($n_a = 6$).

Trajectories with different durations—2, 1.5, 0.75, and 0.5 s—were selected, resulting in progressively higher speeds. The results of this experiment are depicted in Fig.~\ref{fig_Exp4}. The position of the moving endpoint is accurately controlled even under highly dynamic conditions, with steady-state average and maximum position errors of $0.35$ and $1.3$~cm, respectively, while the orientation is maintained with very small errors (less than $0.08$~rad variation in the Euler angle components). Moreover, the reconstructed coordinates of the object shape are shown in Fig. \ref{fig_Exp4}.c, demonstrating that the proposed algorithm efficiently estimates the system configuration even during very fast motions.


\textbf{Model-free versus model-based task-space control in constrained scenarios:} Finally, in a third example, we demonstrate the performance of the controller in a task-oriented setting. Specifically, we consider a scenario in which the robots grasp both endpoints of the object and aim to position its midpoint at a specified target location, where a marker is placed in the DLO for visualization. The system inputs are the 6D wrenches applied at both endpoints ($n_a = 12$). Five main goal locations for the midpoint are defined to accomplish this task in a constrained scenario. In this case, a PD controller with compensation (\ref{eq:PD_controller_comp+}) is implemented as a low-level controller in our approach. Moreover, a model-free approach is used to compare our method with an alternative strategy that does not account for the deformation model. This strategy estimates the initial shape of the object and attempts to control both endpoints by assuming the object to be rigid and non-deformable, in order to position the midpoint at the desired locations. In this case, a baseline PD controller is used to control the actuated coordinates. The frame sequence of the robots in this experiment is shown in Fig. \ref{fig_Exp1_base}.

Based on the solutions obtained from the optimization problem (\ref{eqn:optimization_problem}) for both our approach and the model-free reference, we experimentally assess the controller by measuring the midpoint position, $\bm{p}_{m}$, after moving the robots to each configuration, as illustrated in the example image sequences in Fig.~\ref{fig_Exp1_3}. The results clearly show that the midpoint reachability with our approach is significantly expanded compared to the model-free reference case. This improvement is expected, as explicitly accounting for object deformation enables reaching positions that are not attainable under the rigid-object assumption. 

Finally, Fig.~\ref{fig_Exp1_2} reports the forces and torques applied by the two robots to regulate the object pose and drive its midpoint to the desired position. As illustrated in the snapshots of Fig.~\ref{fig_Exp1_base}, the model-free approach attempts to position the midpoint by rotating and translating both endpoints without deforming the DLO. However, in the constrained scenario, this strategy does not achieve the task with sufficient accuracy. In contrast, our approach explicitly accounts for deformation and determines how to coordinate the motion of both endpoints to deform the object appropriately. This behavior is also evident in the force and torque plots, where the force component along the $X$ direction must increase for the right-side robot ($u_{4}$) and decrease for the left-side robot ($u_{10}$). This asymmetric force modulation enables accurate positioning of the robot endpoints (the actuated coordinates) to compensate for the elastic forces of the object while compressing it. By measuring the mean endpoint error over all evaluated points, we observe a reduction in the average positioning error compared to the model-free reference, from an average of $3.5$~cm with a maximum error of $8.8$~cm to an average of $1.5$~cm and a maximum error of $2$~cm. Compared to the model-free approach, our method benefits from both an explicit understanding of object deformation, which expands the reachable workspace, and accurate knowledge of the object model, which enables more precise positioning of the grasped endpoints, thereby improving task accuracy and efficiency. Finally, Table \ref{tab:final_results} summarizes the steady-state task-level errors obtained in both simulation and experimental scenarios for the model-free and model-based approaches. 

\begin{figure}[H]
    \centering
    \begin{subfigure}[t]{0.45\textwidth}
        \centering
        \includegraphics[width=1\linewidth]{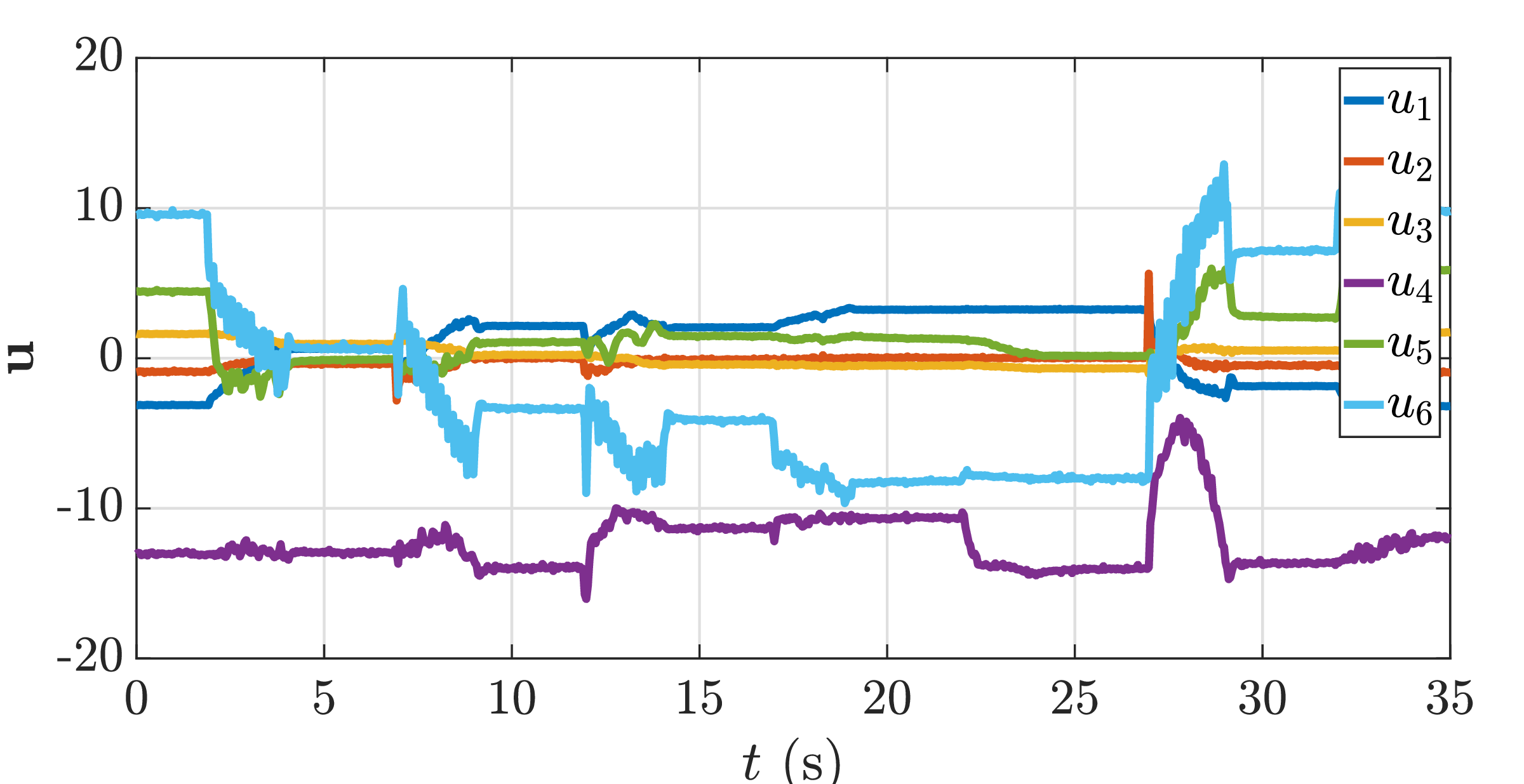}
        \caption{}
    \end{subfigure}
    \centering
    \begin{subfigure}[t]{0.45\textwidth}
        \centering
        \includegraphics[width=1\linewidth]{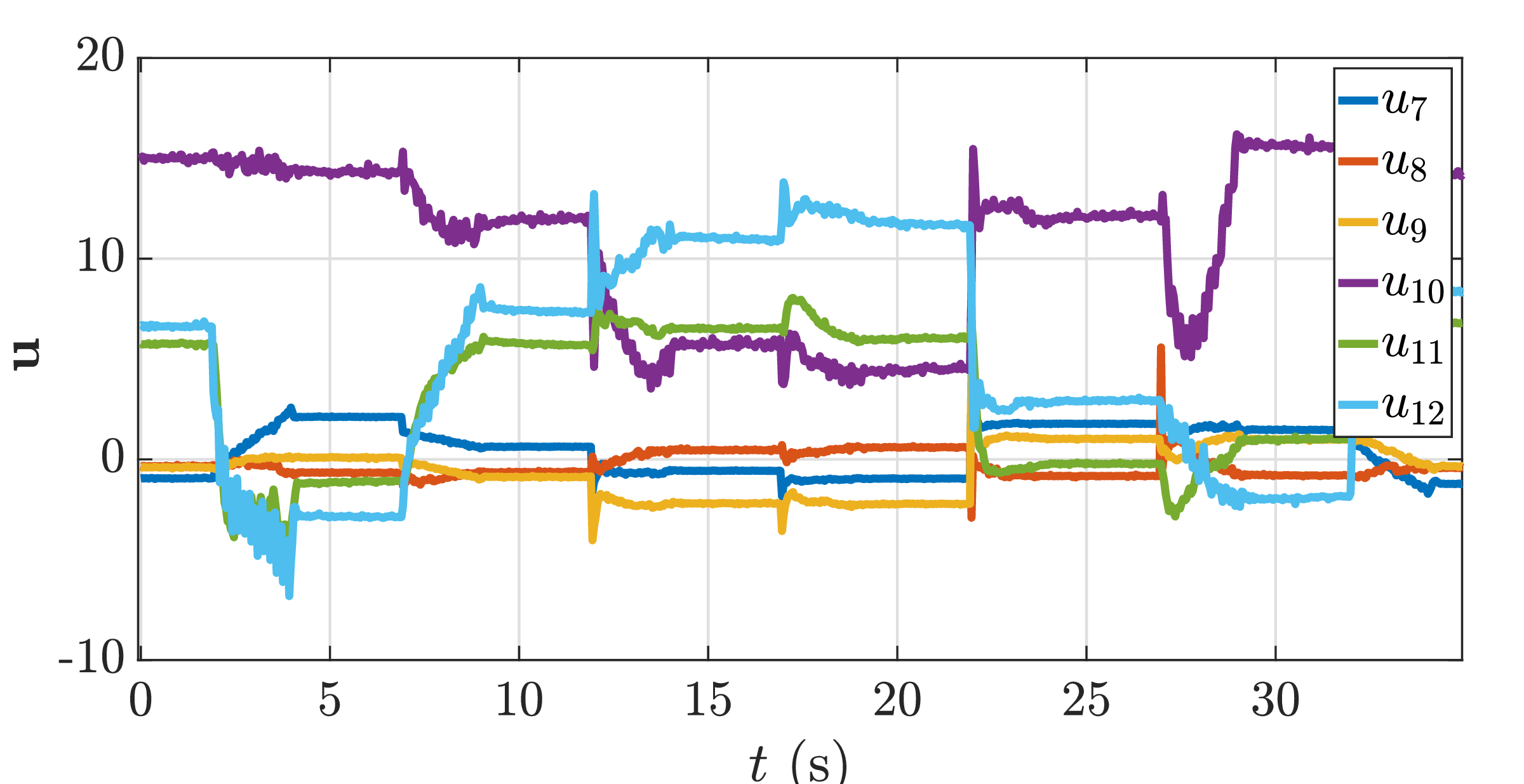}
        \caption{}
    \end{subfigure}
    \begin{subfigure}[t]{0.45\textwidth}
        \centering
        \includegraphics[width=\linewidth]{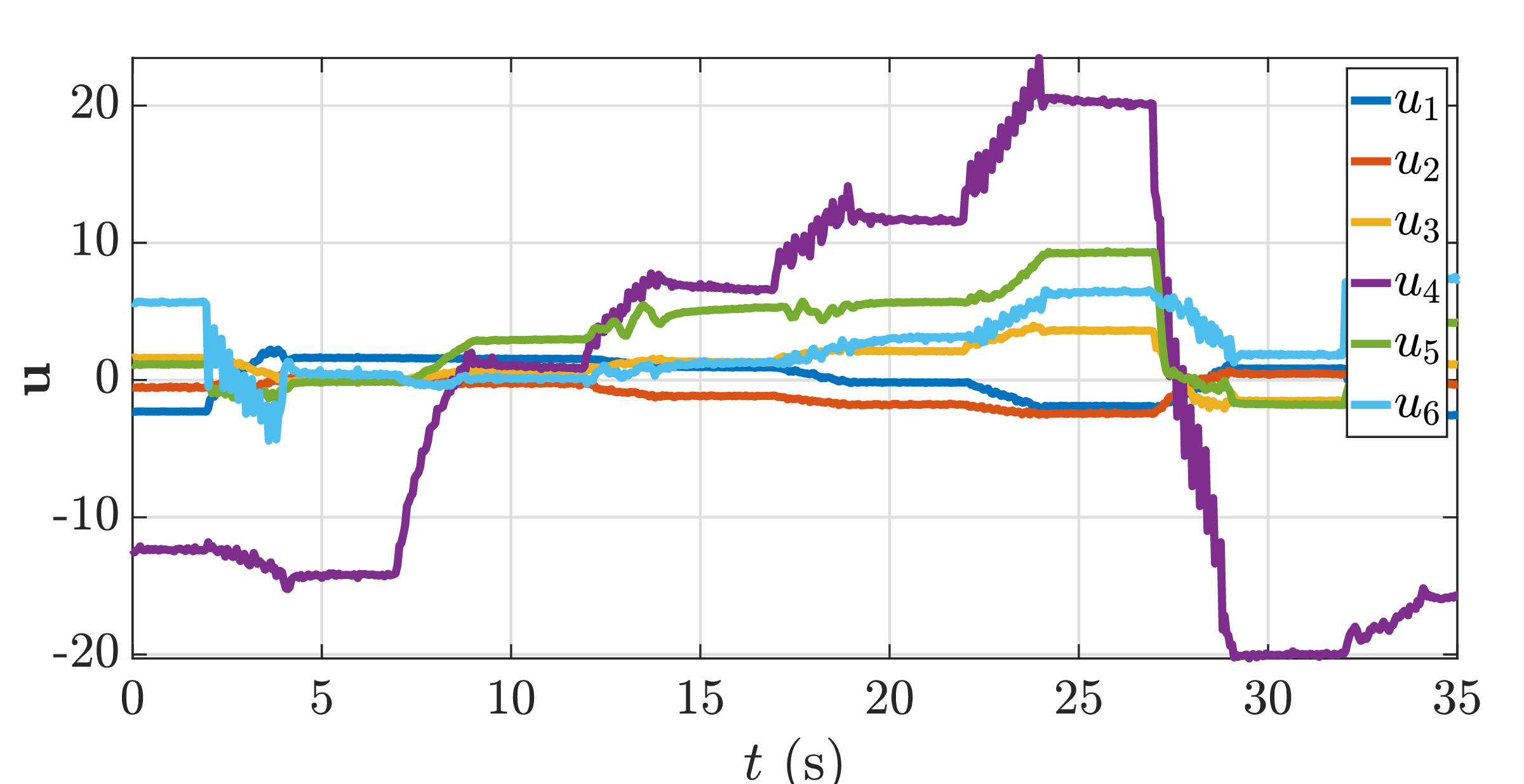}
        \caption{}
    \end{subfigure}
    \centering
    \begin{subfigure}[t]{0.45\textwidth}
        \centering
        \includegraphics[width=\linewidth]{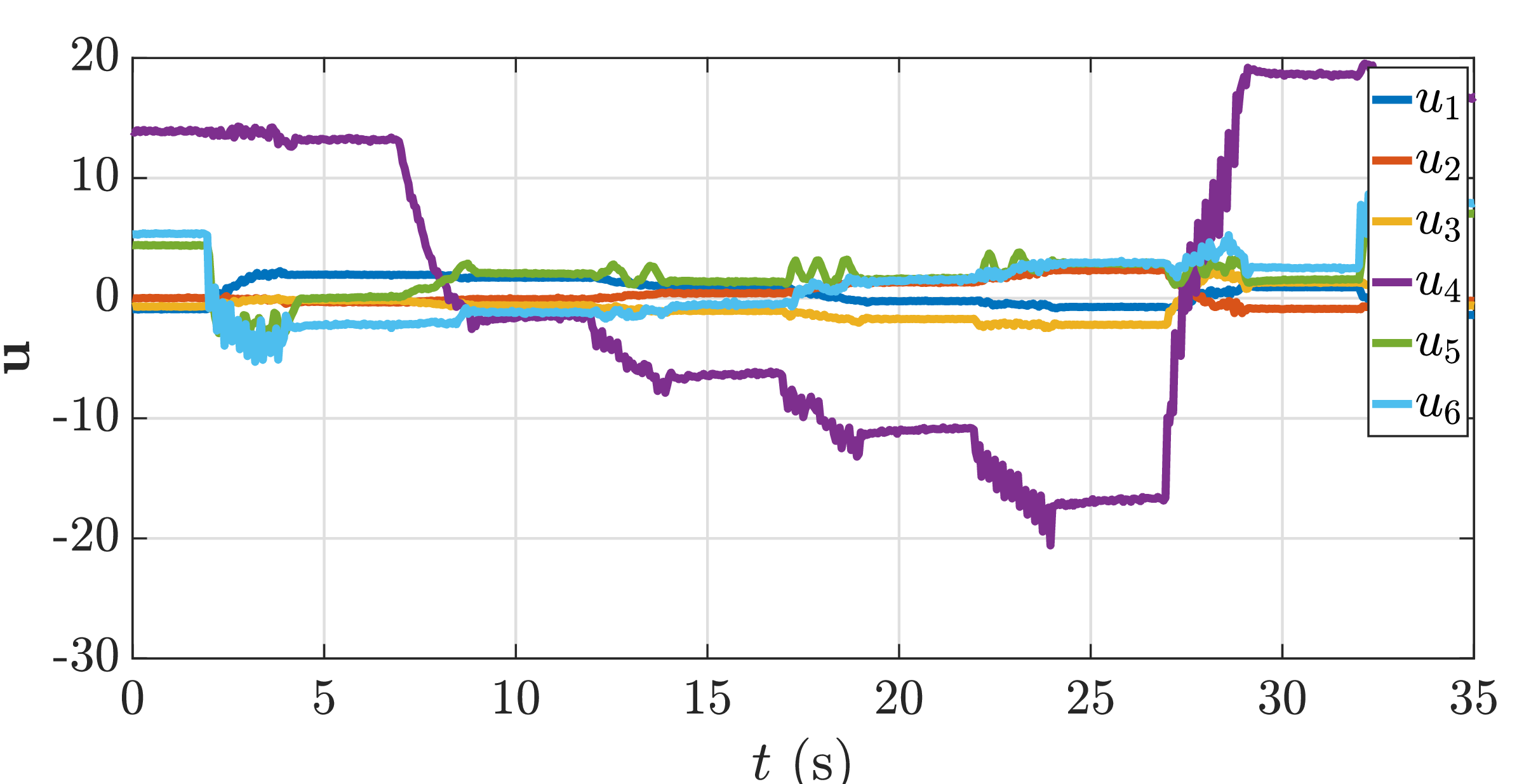}
        \caption{}
    \end{subfigure}
    \caption{\small
    Control inputs for the fourth experimental scenario. Inputs for the model-free approach are shown in (a)–(b), while those for the proposed approach are shown in (c)–(d). Here, $\bm{u}_{1}-\bm{u}_{6}$ represent the components of the 6D wrench applied by the right robot, and $\bm{u}_{7}-\bm{u}_{12}$ represent those applied by the left robot.
    }
    \label{fig_Exp1_2}
\end{figure}

\begin{figure}[H]
    \centering  
    \begin{subfigure}[t]{\columnwidth}
        \centering
        \includegraphics[width=0.95\linewidth]{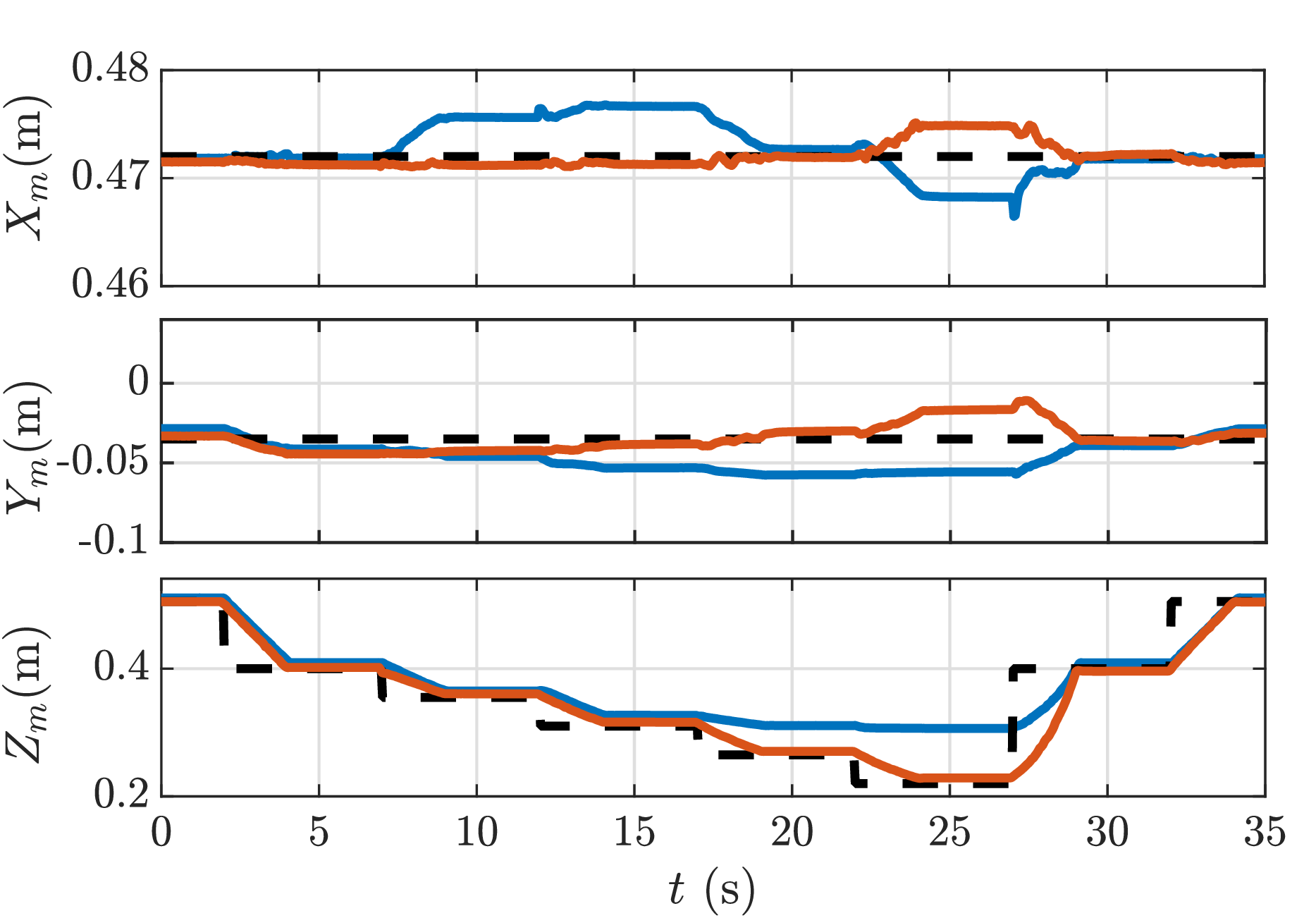}
        \caption{}
    \end{subfigure}
    \begin{subfigure}[t]{\columnwidth}
        \centering
        \includegraphics[width=0.95\linewidth]{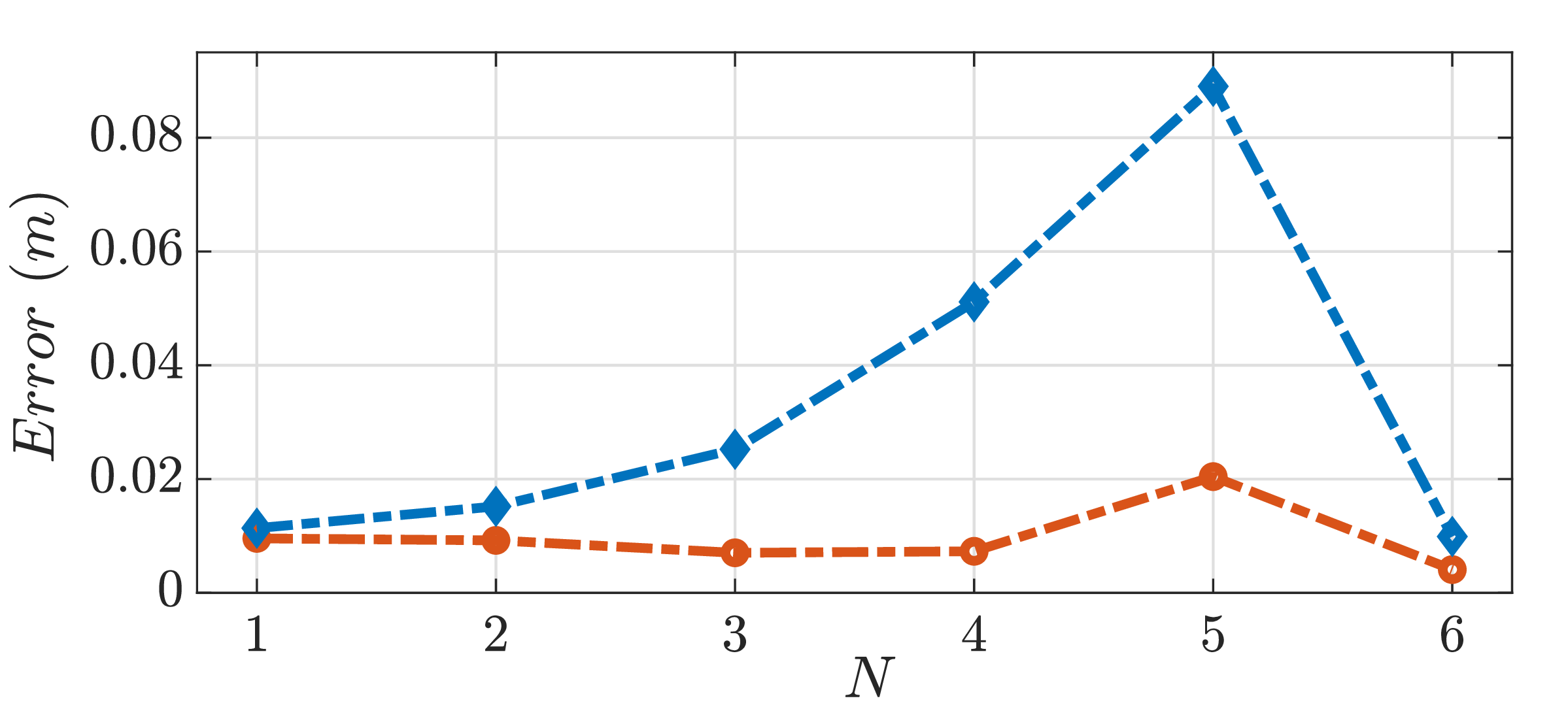}
        \caption{}
    \end{subfigure}

    \caption{\small
    3D position components of the object midpoint for both approaches, together with a comparison of the positional errors.}
    \label{fig_Exp1_3}
\end{figure}

\begin{table}[h]
\centering
\scriptsize
\setlength{\tabcolsep}{12pt} 
\renewcommand{\arraystretch}{0.9} 
\caption{Summary of steady state errors obtained at the task level, when using either a model-free approach (which assumes the object will not deform) and our model-based one to solve the high-level part of the task (left side in Fig. \ref{fig:2_2}). In simulation (Fig. \ref{fig:Sim_2}), the error corresponds to the position error of the grasping points, whereas in the experiments (Fig. \ref{fig_Exp1_3}), it corresponds to the midpoint of the DLO. Errors are reported as a percentage of the object length to provide a normalized measurement. We report both the maximum and the mean error under steady-state conditions.}
\begin{tabular}{|c|cc|cc|}
\hline
 & \multicolumn{2}{c|}{Model-free} & \multicolumn{2}{c|}{Model-based} \\
\hline
 & Max (\%) & Mean (\%) & Max (\%) & Mean (\%) \\
\hline
Sim. & 21.6 & 12.6 & 0 & 0 \\
Exp. & 11.57 & 4.61 & 2.63 & 1.97 \\
\hline
\end{tabular}
\label{tab:final_results}
\end{table}

\section{Conclusions}\label{sec:conclusions}

This work presented a model-based closed-loop control framework for the dynamic manipulation of deformable linear objects through multiple contact points. We modeled this class of objects as floating-base rods subject to localized external forces and torques, explicitly capturing dynamic coupling, elastic response, and gravitational effects. Within this formulation, manipulation was naturally posed as a shape regulation problem, with task-space objectives for Cartesian positioning incorporated when needed. Simulations and experiments demonstrated high steady-state accuracy, strong task-level generalization, and real-time performance without training, consistently outperforming model-free baselines. In particular, controllers that neglected object dynamics developed nonzero steady-state errors as stiffness and inertia increased, whereas the proposed model-based formulation achieved zero steady-state error. A key outcome was the robustness and generality of the approach: a single set of identified physical parameters was used across diverse experimental scenarios, underscoring the potential of model-based strategies for reliable manipulation of deformable objects beyond quasi-static regimes.

\vspace{0.5cm}

\hspace{-0.4cm}{\LARGE\textbf{Appendix}}

\begin{appendices}

\section{Collocated form in Lagrangian systems} 
\label{Sec:Collocated_form}

A Lagrangian system of the form (\ref{eq::gendynamics}) can be conveniently reformulated by separating the dynamic equations of the actuated $\bm{\theta_\mathrm{a}} \in \mathbb{R}^{n_\mathrm{a}}$ and unactuated variables  $\bm{\theta_\mathrm{u}} \in \mathbb{R}^{n - n_\mathrm{a}}$, such that $\bm{\theta} = [\bm{\theta_\mathrm{a}}^{T} \ \bm{\theta_\mathrm{u}}^{T}]^{T}$. Finding a more suitable representation of the system dynamics can greatly simplify both analysis and control design through an appropriate transformation of the generalized coordinates. In particular, the system can be expressed in a collocated form under a change of coordinates $\bm{\theta}_\mathrm{a} = \bm{c}_\mathrm{a}(\bm{q})$, provided that the following assumption holds:

\textit{Integrability assumption:} For all $\bm{q} \in \mathcal{M}$, there exists a function $\bm{c}_\mathrm{a}(\bm{q}) : \mathcal{M} \to \mathbb{R}^m$ such that

\begin{equation}
\label{eq::gendynamics_2}
\bm{J}_{\bm{c}_\mathrm{a}}(\bm{q}) = \frac{\partial \bm{c}_\mathrm{a}}{\partial \bm{q}} = \bm{A}^{T}(\bm{q})
\end{equation}

The integrability assumption implies that each column of $\bm{A}(\bm{q})$ is the gradient of a scalar function of the configuration variables. Under this condition, the Lagrangian system (\ref{eq::gendynamics}) is said to be collocated, as it admits the following collocated representation.

In the underactuated case, the coordinates $\bm{\theta_\mathrm{u}}$ are selected to complete the coordinate transformation $\bm{\theta} = \bm{c}(\bm{q})$ such that the full Jacobian is guaranteed to be full rank.

\begin{equation}
\label{eq::change}
\bm{J}_c(\bm{q}) = \left[\frac{\partial \bm{\theta_\mathrm{a}}}{\partial \bm{q}} \ \ \frac{\partial \bm{\theta_\mathrm{u}}}{\partial \bm{q}} \right] ;
\end{equation}

\noindent where $\bm{J_\mathrm{c}}$ is nonsingular at $\bm{q}$, thus allowing the coordinate change to be well defined. This procedure is described in detail in \cite{pustina2024input}. After performing the change of coordinates, the decoupled dynamics of the system can be expressed as:

%

\begin{equation}\label{eq_decoupling}
\begin{aligned}
\begin{bmatrix}
\bm{M}_\mathrm{aa} & \bm{M}_\mathrm{au} \\
\bm{M}_\mathrm{ua} & \bm{M}_\mathrm{uu}
\end{bmatrix}
\begin{bmatrix}
\ddot{\bm{\theta}}_\mathrm{a} \\
\ddot{\bm{\theta}}_\mathrm{u}
\end{bmatrix}
&+
\begin{bmatrix}
\bm{C}_\mathrm{aa} & \bm{C}_\mathrm{au} \\
\bm{C}_\mathrm{ua} & \bm{C}_\mathrm{uu}
\end{bmatrix}
\begin{bmatrix}
\dot{\bm{\theta}}_\mathrm{a} \\
\dot{\bm{\theta}}_\mathrm{u}
\end{bmatrix}
+
\begin{bmatrix}
\bm{G}_\mathrm{a} \\
\bm{G}_\mathrm{u}
\end{bmatrix} \\
+
\begin{bmatrix}
\bm{K}_\mathrm{aa} & \bm{K}_\mathrm{au} \\
\bm{K}_\mathrm{ua} & \bm{K}_\mathrm{uu}
\end{bmatrix}
\begin{bmatrix}
\bm{\theta}_\mathrm{a} \\
\bm{\theta}_\mathrm{u}
\end{bmatrix}
&+
\begin{bmatrix}
\bm{D}_\mathrm{aa} & \bm{D}_\mathrm{au} \\
\bm{D}_\mathrm{ua} & \bm{D}_\mathrm{uu}
\end{bmatrix}
\begin{bmatrix}
\dot{\bm{\theta}}_\mathrm{a} \\
\dot{\bm{\theta}}_\mathrm{u}
\end{bmatrix}
=
\begin{bmatrix}
\bm{u} \\
\bm{0}
\end{bmatrix}.
\end{aligned}
\end{equation}

\end{appendices}




\section*{Acknowledgements}
The work was supported in part by the Dutch Research Foundation (NWO) through the VENI grant ROSES 20297 and in part by the European Union (ERC, RIPLEY, 101165078).

\bibliographystyle{agsm}
\bibliography{reference}

\end{document}